\documentclass[10pt,journal]{IEEEtran}  %
\usepackage{times}
\usepackage{cite}
\usepackage{hyperref}
\hypersetup{
    colorlinks=true,
    linkcolor=blue,
    filecolor=magenta,      
    urlcolor=cyan,
    }
\usepackage{amsmath,amssymb,amsfonts}
\usepackage{mathtools}
\usepackage{comment}
\usepackage{graphicx}
\usepackage{textcomp}
\usepackage{xcolor}
\usepackage{float}
\usepackage{soul}
\usepackage{enumitem}
\usepackage{cancel}
\usepackage{algorithm}
\usepackage{algpseudocode}
\usepackage{amsthm}
\usepackage{amsmath}
\usepackage{wrapfig}
\usepackage{color, colortbl}
\usepackage{cleveref}
\usepackage{stackengine}

\definecolor{LightCyan}{rgb}{0.4, 0.9, 0.8}
\usepackage[font=small,skip=4pt]{caption} %

\DeclarePairedDelimiter{\floor}{\lfloor}{\rfloor}

\newtheorem{theorem}{Theorem}

\newtheorem{proposition}{Proposition}

\crefname{section}{Sec.}{Secs.}
\Crefname{section}{Sec.}{Secs.}
\Crefname{figure}{Fig.}{Figs.}
\usepackage{soul}
\usepackage{acro}

\DeclareAcronym{imu}{
  short=IMU,
  long=Inertial Measurement Unit,
}

\DeclareAcronym{vio}{
  short=VIO,
  long=Visual Inertial Odometry,
}

\DeclareAcronym{vfm}{
  short=VFM,
  long=Visual Foundation Model,
}

\DeclareAcronym{slam}{
  short=SLAM,
  long=Simultaneous Localization and Mapping,
}

\DeclareAcronym{msckf}{
  short=MSCKF,
  long=Multi-state Constraint Kalman Filter,
}

\DeclareAcronym{cpu}{
  short=CPU,
  long=Central Processing Unit,
}

\DeclareAcronym{mcts}{
  short=MCTS,
  long=Monte Carlo Tree Search,
}

\DeclareAcronym{scp}{
  short=SCP,
  long=Sequential Convex Programming,
}

\DeclareAcronym{dof}{
  short=DOF,
  long=Degrees of Freedom,
}

\DeclareAcronym{dnn}{
  short=DNN,
  long=Deep Neural Network,
}

\DeclareAcronym{cp}{
  short=CP,
  long=Capture Point,
}

\DeclareAcronym{lip}{
  short=LIP,
  long=Linear Inverted Pendulum,
}

\DeclareAcronym{cots}{
  short=COTS,
  long=commercial-off-the-shelf,
}

\DeclareAcronym{canfd}{
  short=CAN-FD,
  long=Controller Area Network Flexible Data-Rate,
}

\DeclareAcronym{mpc}{
  short=MPC,
  long=Model Predictive Control
}

\DeclareAcronym{rl}{
  short=RL,
  long=Reinforcement Learning
}

\DeclareAcronym{linc}{
  short=LINC,
  long=Learning Introspective Control
}

\DeclareAcronym{icr}{
  short=ICR,
  long=Instantaneous Center of Rotation
}

\DeclareAcronym{agv}{
  short=AGVs,
  long=Autonomous Ground Vehicles
}

\DeclareAcronym{cast}{
  short=CAST,
  long=Center for Autonomous Systems and Technologies
}

\DeclareAcronym{rmse}{
  short=RMSE,
  long=Root Mean Square Error
}

\newcommand{\N}{\mathbb{N}}
\newcommand{\R}{\mathbb{R}}
\DeclareMathOperator*{\E}{\mathbb{E}}
\renewcommand{\th}{^{\mathrm{th}}}
\newcommand{\simplex}[1]{\triangle #1}
\newcommand{\uniform}{\mathcal{U}}
\DeclareMathOperator*{\argmin}{argmin\ }

\newcommand{\cE}{\boldsymbol{\mathcal{E}}}
\newcommand{\mA}{\mathbf{A}_\mathrm{n}}
\newcommand{\mB}{\mathbf{B}}  
\newcommand{\mC}{\mathbf{C}}
\newcommand{\mK}{\mathbf{K}} 
\newcommand{\mM}{\mathbf{M}}
\newcommand{\mQ}{\mathbf{Q}} 
\newcommand{\mR}{\mathbf{R}} 

\newcommand{\mgamma}{\mathbf{\Gamma}} 
\newcommand{\vp}{\mathbf{p}}
\newcommand{\vq}{\mathbf{q}}
\newcommand{\vs}{\mathbf{s}}
\newcommand{\vtheta}{\boldsymbol{\theta}}
\newcommand{\vvu}{\mathbf{u}}
\newcommand{\vu}{\mathbf{u}}
\newcommand{\vv}{\mathbf{v}}
\newcommand{\vw}{\mathbf{w}}
\newcommand{\vx}{\mathbf{x}}
\newcommand{\vy}{\mathbf{y}}

\newcommand{\mH}{\mathbf{H}}

\newcommand{\vdelta}{\boldsymbol{\delta}}
\newcommand{\mphi}{\boldsymbol{\Phi}}
\newcommand{\mW}{\mathbf{W}}

\newcommand{\hai}{\hat{\theta}_i}
\newcommand{\dhai}{\dot{\hat{\theta}}_i}
\newcommand{\tai}{\tilde{\theta}_i}
\newcommand{\dtai}{\dot{\tilde{\theta}}_i}
\newcommand{\Lip}{{\operatorname{Lip}}}

\newcommand{\dist}{\mathbf{d}}
\newcommand{\window}{{\hat \ell}}
\newcommand{\namepaper}{MAGIC\textsuperscript{VFM}}

\newcommand{\revision}[1]{\textcolor{black}{#1}}
\DeclarePairedDelimiter\abs{\lVert}{\rVert}

\newcommand{\revisiontwo}[1]{\textcolor{black}{#1}}

\title{\namepaper~-Meta-learning Adaptation for Ground Interaction Control
with Visual Foundation Models}

\author{Elena Sorina Lupu$^{*}$, Fengze Xie$^*$, James A. Preiss, Jedidiah Alindogan, Matthew Anderson, Soon-Jo Chung\thanks{$^*$Co-first authors. This work was supported by DARPA \ac{linc} and Amazon AI for Science.}}

\begin{document}

\maketitle

\begin{abstract} 
\revision{
    Control of off-road vehicles is challenging due to the complex dynamic interactions with the terrain.
    Accurate modeling of these interactions is important to optimize driving performance, but the relevant physical phenomena, \revisiontwo{such as slip,} are too complex to model from first principles.
    Therefore, we present an offline meta-learning algorithm to construct a rapidly-tunable model of residual dynamics and disturbances.
    Our model processes terrain images into features using a visual foundation model (VFM), then maps these features and the vehicle state to an estimate of the current actuation matrix using a deep neural network (DNN).
    We then combine this model with composite adaptive control
    to modify the last layer of the DNN in real time,
    accounting for the remaining terrain interactions not captured during offline training.
    We provide mathematical guarantees of stability and robustness for our controller,
    and demonstrate the effectiveness of our method through simulations and hardware experiments with a tracked vehicle and a car-like robot.
    We evaluate our method outdoors on different slopes with varying slippage and actuator degradation disturbances, and
    compare against an adaptive controller that does not use the VFM terrain features.
    We show significant improvement over the baseline in both hardware experimentation and simulation.
}
\end{abstract}

\section{Introduction}
\ac{agv} are gaining popularity across numerous domains including agriculture \revisiontwo{applications}~\cite{agriculture_1, Gonzalez-De-Santos20, agriculture_2}, wilderness search and rescue missions~\cite{delmerico_active_2017, kashino_aerial_2020, qin_autonomous_2019, gadekar_rakshak_2023}, and planetary exploration~\cite{doi:10.1126/scirobotics.adi3099}. In many of these scenarios, the \ac{agv} operate on rugged surfaces where the ability to follow a desired trajectory \revision{degrades}.
To reliably operate in these environments with minimal human intervention, \ac{agv} must understand the environment and adapt to it in real time.
Slippage is one of the primary challenges \revision{encountered by} ground vehicles while operating on loose terrain.
For rovers exploring other planets, slippage can slow down their progress and even halt their scientific objectives.
For instance, the Opportunity rover recorded significant slippage and sinking of its wheels during the Mars \revisiontwo{D}ay 2220~\cite{opportunity_mer} while traversing sand ripples.
During its climb, the slip, calculated based on visual odometry~\cite{maimone_slip}, was high, and thus the drive was halted and \revision{rerouted} by the ground operators.

\begin{figure}[t]
    \centering\includegraphics[width=\columnwidth]{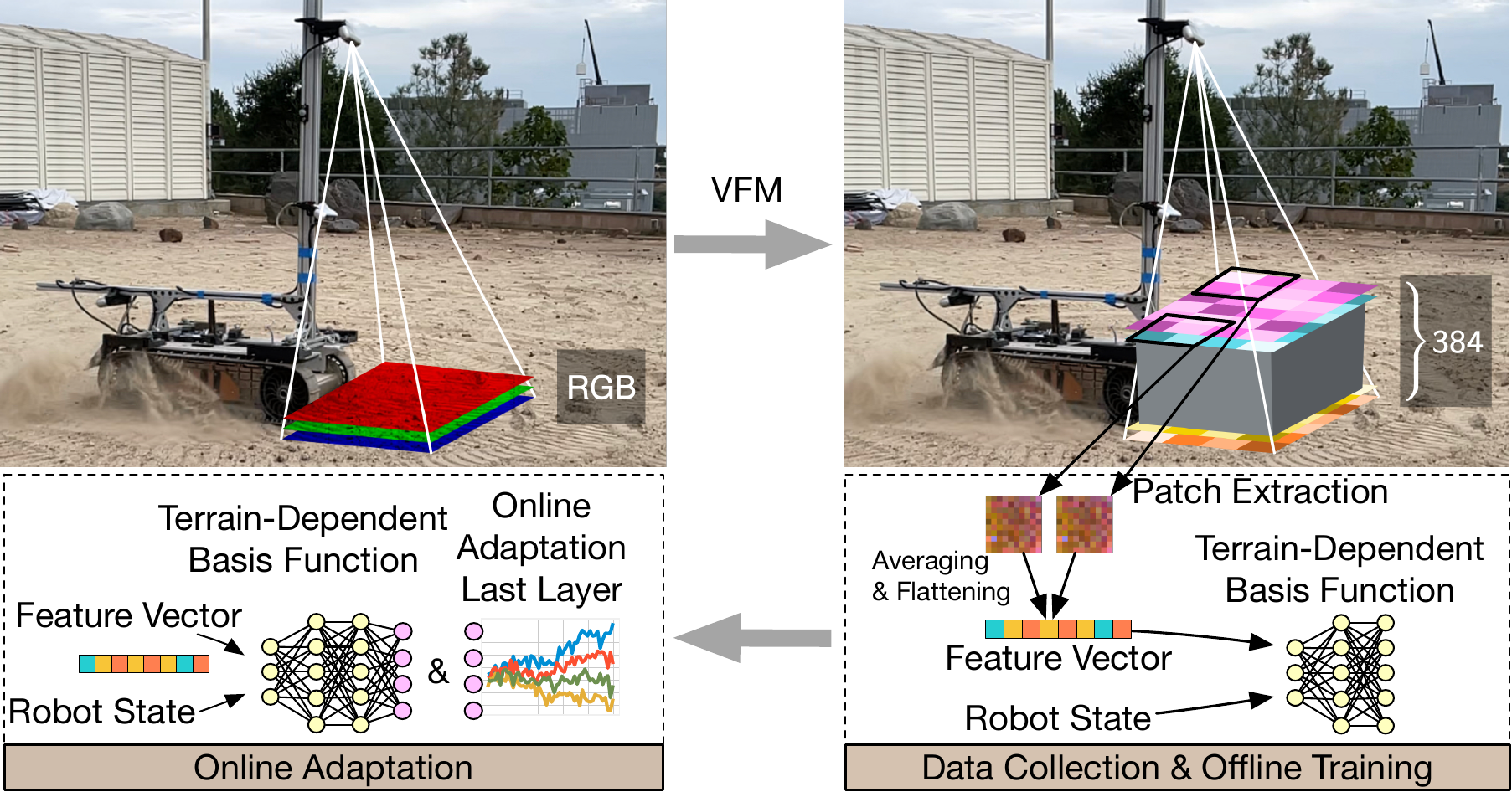}
\caption{MAGIC\textsuperscript{VFM}: 
\revision{An offline meta-learning algorithm to build a  residual dynamics and disturbance model using both Visual Foundation Models (VFM) and vehicle states. This model is integrated with composite adaptive control to adapt to changes in both the terrain and vehicle dynamics conditions in real time. See \url{https://youtu.be/sxM73ryweRA}}
\label{fig:concept_overview}
}

    \label{fig:gvrbot}
\end{figure}

To better understand the effects of terradynamics, researchers have designed sophisticated models~\cite{wong2001theory} that inform the design, simulation, and control of ground vehicles. 
However, these models have numerous assumptions and are often limited when the vehicles are operated at their performance boundaries (e.g., steering at high speeds and instances of non-uniform resistive forces like stumps and stones).
In addition, designing controllers that consider these complex models is challenging.
For control, kinematic models, such as Dubins, are often employed due to their simplicity and intuitive understanding.
However, these models are not able to capture the complicated dynamics between the vehicle and the ground, nor other disturbances such as internal motor dynamics or wheel or track degradation.
To increase the performance of ground vehicles, more comprehensive models are necessary. 

Controllers that stabilize a ground vehicle and track desired trajectories amidst a variety of disturbances are crucial for achieving optimal vehicle performance. Oftentimes, the bottleneck is not in the controller design per se, but rather in the choice and complexity of the model utilized \revision{by} the controller. 
\revision{Recently}, reinforcement learning (RL) has shown significant promise in facilitating the development of efficient controllers through experiential learning~\cite{chua_deep_2018,peng2021linear,tobin2017domain}.
The combination of
meta-learning~\cite{santoro_meta-learning_2016,meta-learning-survey,metalearning,nagabandi2019learning,clavera2018model,goodfellow2014generative,ganin2016domain,mckinnon2021meta} and adaptive control~\cite{slotine1991applied,ioannou1996robust,krstic1995nonlinear,narendra2012stable,O_Connell_2022,chen1995adaptive,farrell_adaptive_2006,nakanishi_locally_2002,pravitra_L1-adaptive_2020} demonstrates considerable potential in accurately estimating unmodeled dynamics, \revisiontwo{efficiently} addressing domain shift challenges and real-time adaptation to new environments~\cite{nagabandi2019learning, song2020rapidly, Belkhale_2021, 9369887}. 
Despite this progress, incorporating a suitable model into a controller/policy is still an active area of research,
especially when combined with theoretical and safety guarantees.

Learning sophisticated unmodeled dynamics based only on a limited set of vehicle states is ill-posed given that the operating environment is infinite dimensional.
To accurately represent a complete dynamics model, including learned residual terms for control, it is imperative to leverage as much information about the environment as possible.
For instance, visual information can inform the model about the type of terrain in which the vehicle is operating.
Previous work includes segmentation-based models that assign a discrete terrain type to each area in the image.
This information is further employed in planning and control~\cite{rothrock_spoc_2016}.
However, in off-road applications, categorizing terrains into a limited number of classes such as snow, mud, sand, etc. is not sufficient. There are infinite subcategories within each terrain type, each presenting distinct effects on the vehicle.
In addition, two terrains can appear similar \revisiontwo{yet} induce different dynamic behaviors on the robot (e.g., deep and shallow sand). Therefore, finding the \revisiontwo{most accurate and} robust representation of the environment is indispensable for vehicle control over complex terrain.

\subsection{Contributions}
To address these limitations, we present \namepaper{} (\textbf{M}eta-learning \textbf{A}daptation for \textbf{G}round \textbf{I}nteraction \textbf{C}ontrol with \textbf{V}isual \textbf{F}oundation \textbf{M}odels), an approach that integrates a \ac{vfm} with meta-learning and \revision{composite} adaptive control, thereby enabling ground vehicles to navigate and adapt to complex terrains in real time.
Our method is \revisiontwo{well-}suited for any ground vehicle equipped with the following: 1) sensors \revision{to measure the} internal \revision{robot} state, 2) exteroceptive sensors that can capture the terrain such as cameras, 3) the availability of a pre-trained \ac{vfm}, and 4) the necessary computation hardware to evaluate the \ac{vfm} in real-time.
Our contributions are:
\begin{itemize}[noitemsep,  leftmargin=*]
    \item the first stable \textit{learning-based adaptive controller} that incorporates \textit{visual foundation models} for terrain adaptation \revisiontwo{in real time};
    \item an offline meta-learning algorithm that uses continuous trajectory data to train and learn the terrain disturbance as a function of visual terrain information and vehicle states;
    \item mathematical guarantees of exponential stability \revision{and robustness against model errors} for adaptive control with visual and state input that works in conjunction with our  meta-learning offline training algorithm;
    \item the development of a position, attitude, and velocity tracking control formulation with the control influence matrix adaptation that can handle a variety of other perturbations in real-time such as unknown time-varying track or motor degradation and arbitrary time-varying disturbances.
\end{itemize}

We validate the effectiveness of our method both through simulation and in hardware \revision{on two heterogeneous robotic platforms}, demonstrating its performance outdoors, on slopes with different slippage, as well as under track degradation disturbances.

\subsection{Paper Organization}
\revision{The} paper is organized as follows:~\Cref{sec:related_work} provides a review of existing literature.
In~\Cref{sec:terrain_informed_adaptive}, our offline meta-learning algorithm and the online adaptation algorithm are presented.
~\Cref{sec:modeling} \revision{presents} the \revision{two different vehicle} models \revision{and} the adaptive controllers with stability and robustness guarantees.
In~\Cref{sec:interprebility}, we analyze the \ac{vfm} output for terrain, particularly in relation to our learning-based adaptive controller.
In~\Cref{sec:simulation}, we validate the algorithm using simulation and continue with experimental validation in~\Cref{sec:experimental} and concluding remarks in~\Cref{sec:conclusions}.

\subsection{Notation}\label{sec:notation}
Unless otherwise noted,
all vector norms are Euclidean
and all matrix norms are the Euclidean operator norm.
We denote the floor operator by $\floor{\cdot}$.
Given $\mathbf{A} \in \mathbb{R}^{n \times m \times p}$ and $\mathbf{b} \in \mathbb{R}^{p}$, the notation $\left(\mathbf{A}\mathbf{b}\right)$ is defined as 
$\sum_{i=1}^{p} \mathbf{A}_i b_i$.
The notation $\Vert \mathbf{x} \Vert_\mathbf{P}$ for positive semi-definite matrix $\mathbf{P}$ defines the weighted inner product $\sqrt{\mathbf{x}^\top \mathbf{P}\mathbf{x}}$.
For a function $f : X \mapsto Y$ where $X$ and $Y$ are metric spaces with metrics $d_X$ and $d_Y$, we define
$\| f \|_\Lip = \max_{x, x' \in X} d_Y(f(x), f(x')) / d_X(x, x')$.
For a measurable set $X$, we denote the set of probability measures on $X$ by $\simplex{X}$,
and if a uniform distribution on $X$ exists, we denote it by $\uniform X$.
When clear from context, we overload the notation $[i, j]$ to denote the integer sequence $i, \dots, j$. \revision{All matrices and vectors are written in bold.}

\section{Related Work}\label{sec:related_work}

The term \emph{meta-learning}, first coined in~\cite{schmidhuber:1987:srl}, most often refers to learning protocols in which there is an underlying set of related learning tasks/environments, and the learner leverages data from previously seen tasks to adapt rapidly to a new task~\cite {metalearning, shi2021metaadaptive, meta-learning-survey, xiao2023safe}.
The goal is to adapt more rapidly than would be possible for a standard learning algorithm presented with the new task in isolation.
In robotics, meta-learning has been used to accurately adapt to highly dynamic environments~\cite{nagabandi2019learning, song2020rapidly, Belkhale_2021, 9369887}.
Online meta-learning~\cite{online_meta_learning, O_Connell_2022, meta_learning_online_corr, pmlr-v97-finn19a} includes two phases: offline meta-training and online adaptation.
In the offline phase, the goal is to learn a model that performs well across all environments using a meta-objective.
Given limited real-world data, the online adaptation phase aims to use online learning, such as adaptive control~\cite{slotine1991applied}, to adapt the offline-learned model to a new environment in real time.

Some examples of meta-learning algorithms from \revisiontwo{the} literature are Model-Agnostic Meta-Learning (MAML)~\cite{metalearning} with its online extension~\cite{pmlr-v97-finn19a}, Meta-learning via Online Changepoint Analysis (MOCA)~\cite{DBLP:journals/corr/abs-1912-08866}, and Domain Adversarially Invariant Meta-Learning (DAIML) used in Neural-Fly~\cite{O_Connell_2022}.
For task-centered datasets, MAML~\cite{metalearning} trains the parameters of a model to achieve optimal performance on a new task with minimal data, by updating these parameters through one or more gradient steps based on that task's dataset.
In continuous problems, tasks often lack clear segmentation, resulting in the agent being unaware of task transitions. 
Therefore, MOCA~\cite{DBLP:journals/corr/abs-1912-08866} proposes task unsegmented meta-learning via online changepoint analysis.
DAIML~\cite{O_Connell_2022} proposes an online meta-learning-based approach 
\revision{where} a shared representation is learned offline \revision{(for example, using data from different wind conditions for a quadrotor)}, with a linear, low-dimensional part updated online through adaptive control.

Our method builds on the previous work~\cite{O_Connell_2022} on the integration of adaptive control and offline meta-learning to build a comprehensive model for ground vehicles. We develop a meta-learning algorithm that uses continuous trajectories from a robot driving on different terrains to learn a representation of the dynamics residual common across these terrains.
This representation is a \ac{dnn} that encodes the terrain information through vision, together with a set of linear parameters that adapt online at runtime.
These linear parameters can be interpreted as the last layer of the \ac{dnn}~\cite{O_Connell_2022} and \revisiontwo{are} terrain independent but encapsulate the remaining disturbances not captured during training.

\subsection{Embedding Visual Information in Classical Control and Reinforcement Learning}

One of the early works on including visual information for control is visual servoing~\cite{espiau1992new}, a technique mainly used for robot manipulation.
Recently, vision-based reinforcement learning (VRL) has demonstrated the capability to control agents in simulated environments\cite{world_models}, as well as robots in real environments, \revision{with applications to ground robots}~\cite{doi:10.1126/scirobotics.abk2822, pmlr-v205-wu23c, cheng2023extreme} \revision{and \revisiontwo{robot} manipulation} \cite{JMLR:v17:15-522, Florence2019SelfSupervisedCI, visualforesight}.
This capability is achieved by leveraging high-fidelity models in robotics simulators~\cite{todorov2012mujoco, 8255551}, \revision{imitation learning from human demonstrations or} techniques to bridge the sim-to-real gap of the learned policy~\cite{tobin2017domain}.
Nevertheless, a limitation of VRL is that the generated policy remains uninterpretable and does not have safety and robustness guarantees. 
To address the uninterpretability aspect, recent advancements in Inverse Reinforcement Learning (IRL) offer promising \revisiontwo{algorithms} for interpreting terrain traversability as a reward map, thus enhancing the understanding of the environments~\cite{Gan_2022, frey2023fast}.
Despite progress in combining vision with \ac{rl}, incorporating a suitable terrain model into a policy is still an open area of research, especially when combined with theoretical guarantees and safety properties.

To this end, we derive a nonlinear adaptive tracking controller for \ac{agv} that uses a learned ground model with vision information \revisiontwo{incorporated} in the control influence matrix.
Our method processes camera images that are then passed through a \ac{vfm} to synthesize the relevant features. These features, together with the robot's state, are incorporated into the ground\revisiontwo{-robot interaction} model learned offline using meta-learning. For this adaptive controller, we prove exponential convergence to a bounded error ball.

\subsection{Visual Foundation Models in Robotics}

A foundation model is a large-scale machine learning model trained on a broad dataset that can be adapted, fine-tuned, or built upon for a variety of applications.
Self-supervised learned \ac{vfm}s, such as Dino and DINOv2~\cite{oquab2023dinov2}, are foundation models that are based on visual transformers~\cite{NIPS2017_3f5ee243, DBLP:journals/corr/abs-2010-11929}.
These models are trained to perform well on several downstream tasks, including image classification, semantic segmentation, and depth estimation.
In robotics, these foundation models are starting to gain popularity in tasks such as image semantic segmentation~\cite{hamilton2022unsupervised, Erni2023-bs}, traversability estimation~\cite{frey2023fast}, \revision{and \revisiontwo{robot} manipulation~\cite{ze2023gnfactor, wang2023d3fields}}. One of their key advantages lies in the robustness against variations in lighting and occlusions \cite{naseer2021intriguing}, as well as their ability to generalize well across different images of the same context.
By consuming raw images as inputs, these self-supervised learning foundation models possess the potential to learn all-purpose visual features if pre-trained on a large quantity of data.

\subsection{Adaptation to Ground Disturbances}

Adaptive control~\cite{slotine1991applied,ioannou1996robust,krstic1995nonlinear,narendra2012stable,O_Connell_2022,chen1995adaptive,farrell_adaptive_2006,nakanishi_locally_2002,pravitra_L1-adaptive_2020,ANNASWAMY202118} is a control method with provable convergence guarantees in which a set of linear parameters \revisiontwo{is} adapted online to compensate for disturbances at runtime.
Typically, these \revision{linear} parameters are multiplied by a basis function, which can be constant (as in the case of integral control), derived from physics~\cite{shi_adaptive_2020}, or represented using Radial Basis Functions (RBFs)~\cite{KADIRKAMANATHAN1995245} or \ac{dnn}s~\cite{O_Connell_2022}.
\revision{First introduced in \cite{slotine1991applied,slotine1989composite}, composite adaptation combines online parameter estimation and tracking-error adaptive control. A rigorous robustness/stability analysis for composite adaptation with a connection to deep meta-learning was first derived in~\cite{O_Connell_2022} for flight control applications.}

Ground vehicles (including cars, tracked vehicles, and legged robots) should be adaptive to changes in the terrain conditions.
This adaptability is essential for optimal performance and safety in diverse environments~\cite{LUA202314175, doi:10.1177/0954407019901083, sombolestan2021adaptive, 976416, petrov2014modeling, khan_fast_2022}.
In~\cite{visca2022deep}, an adaptive energy-aware prediction and planning framework for vehicles navigating over terrains with varying and unknown properties \revisiontwo{was} proposed and demonstrated in simulation.
\cite{metaverse} proposes a deep meta-learning framework for learning a global terrain traversability prediction network that is integrated with a sampling-based model predictive controller, while \cite{probabilistic_traversability_model} develops a probabilistic traction model with uncertainty quantification using both semantic and geometric terrain features. \revision{In~\cite{banerjee2020adaptive}, a meta-learning-based approach to adapt probabilistic predictions of rover dynamics with Bayesian regression is used.
}
\revisiontwo{While these approaches succeeded in developing different models for control and planning, incorporating a suitable model that fits the adaptive control framework is an open area of research.}

\revision{In this paper, w}e establish an adaptive controller that can \revision{handle} a broad range of real-time perturbations, such as unknown time-varying track or motor degradation, under controllability assumptions, arbitrary time-varying disturbances, and model uncertainties.
\revision{This work can be viewed as a generalization and improvement of~\cite{O_Connell_2022} with a new control matrix adaptation method using visual information and improved stability results.}
\revision{Our} adaptive capability acts as an enhancement to the offline trained basis function\revision{, further improving tracking performance under challenging conditions}.

\begin{figure*}[t]\centering\includegraphics[width=0.9\textwidth]{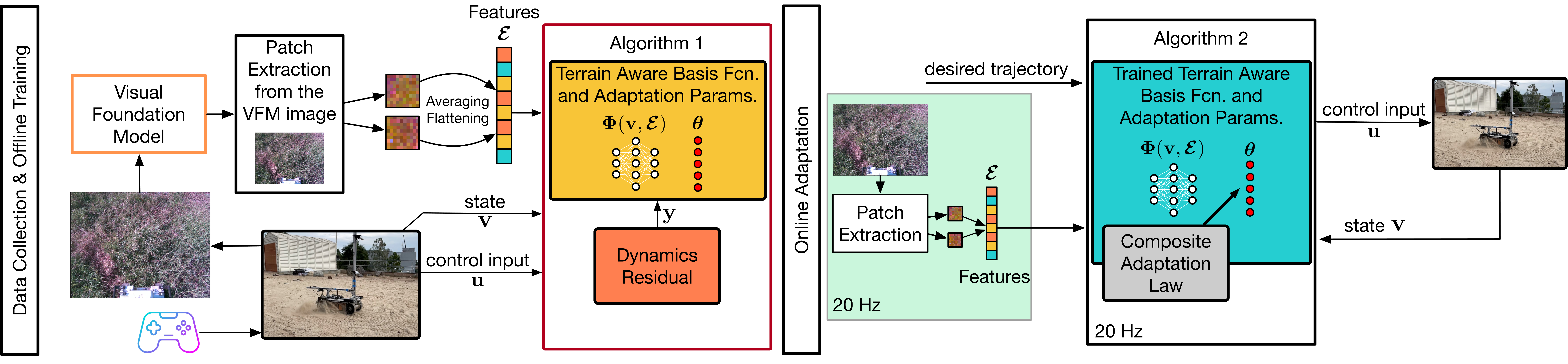}
    \caption{
        Terrain-aware Architecture: offline data collection and training (Algorithm~\ref{alg:training-v2}), followed by real-time adaptive control (Algorithm~\ref{alg:adaptation_at_execution_time}) running onboard the robot.}
    \label{fig:arch}
\end{figure*}

\section{Methods}\label{sec:terrain_informed_adaptive}
\subsection{Residual Dynamics Representation using \ac{vfm}}
\label{sec:model}
We consider an uncertain dynamical system model
\begin{equation}
\label{eq:general_dynamics}
    \dot{\mathbf{x}} = \mathbf{f}(\mathbf{x}, \mathbf{u}, t) + \dist,
\end{equation}
where $\mathbf{x} \in \mathbb{R}^n$ denotes the state,
$\mathbf{u} \in \mathbb{R}^m$ denotes the control input,
$\mathbf{f}\!:\!\mathbb{R}^{n + m + 1} \mapsto \mathbb{R}^n $ denotes the nominal dynamics model,
and $\dist$ is an unknown disturbance that is possibly time-varying and state- and environment-dependent.
Our algorithm approximates $\dist$ by
\begin{equation}
    \dist = \mphi(\mathbf{x},  \mathbf{u}, \cE) \boldsymbol{\theta}(t) + \vdelta, \label{Eq:Phidef_original}
\end{equation}
where
${\boldsymbol{\theta}\revision{(t)} = \left[\theta_1  \ldots \theta_{n_\theta} \right]^\top \in \mathbb{R}^{n_\theta}}$
denotes a \revision{time-varying} vector of linear parameters that are adapted online by our algorithm, $\vdelta $\revision{ ${} \in \mathbb{R}^n$} is a representation error,
and $\cE \in \mathbb{R}^{p}$ is a feature vector representation of the terrain surrounding the robot computed by a \ac{vfm}.
From the perspective of \revision{the adaptive control part of our method}, the precise \revision{form of $\cE$ is not required in our work}.
We can think of $\cE$ as computed by some arbitrary \revision{Lipschitz} function from the robot's sensors to $\R^p$.
We \revisiontwo{provide} details on the particular form of $\cE$ used in our empirical sections (\Cref{sec:interprebility}-\ref{sec:experimental}).
The feature mapping $\mphi(\mathbf{x},  \mathbf{u}, \cE):\mathbb{R}^{n + m + p} \mapsto \mathbb{R}^{n \times n_\theta} $ is learned in the offline training phase of our algorithm.
Our method supports arbitrary parameterized families of continuous functions, but in practice we focus on the case where $\mphi$ is a \ac{dnn}.
In this case, $\boldsymbol{\theta}$ can be regarded as the weights of the last layer of the \ac{dnn}~\eqref{Eq:Phidef_original}, which continuously adapt in real-time.
The online adaptation is necessary in real scenarios, because two \revision{environments (i.e.,} terrains\revision{)} might have the same representation~$\boldsymbol{\mathcal{E}}$, but induce different dynamic behaviors onto the robot, as well as for other types of disturbances not captured in \revision{the feature mapping} $\mphi$.

Further, we assume that the disturbance $\mathbf{d}$ is affine in the control input $\mathbf{u}$\revision{, taking the form}
\begin{equation}
\label{Eq:Phidef}
    \mathbf{d} \approx \mphi^\vw(\mathbf{x},  \cE) \boldsymbol{\theta} \mathbf{u} = \sum_{i=1}^{n_\theta} \theta_i \mphi^\vw_i(\mathbf{x},  \cE) \vu \revision{,}
\end{equation}
\revision{where} the dependence on a parameter vector $\vw$ (the \ac{dnn} weights) is made explicit, and
the basis function has the form ${\mphi^\vw(\mathbf{x}, \cE): \mathbb{R}^{n + p} \mapsto \mathbb{R}^{n \times m \times n_\theta}}$
with the individual matrix-valued components denoted \revisiontwo{by}
$\mphi^\vw_i(\mathbf{x}, \cE): \mathbb{R}^{n  + p} \mapsto \mathbb{R}^{n \times m}$.
The control affine assumption is motivated by two \revision{factors}.
First, in our main application of ground vehicles with desired-velocity inputs,
input-affine disturbances more accurately capture terrain interactions such as
slippage~(\Cref{sec:motivation_control_matrix}), internal dynamics, 
and wheel or track degradations.
Second, \revisiontwo{this assumption} simplifies the exponential convergence proof for our adaptive controller given in~\Cref{theorem1}.
\revision{However, we emphasize that~\Cref{theorem1} can be extended to more general forms of disturbances by input-output stability combined with contraction theory~\cite{chung2013phase, tsukamoto2021contraction}.}
Lastly, we define the dynamics residual
\revision{that is used in both the online and offline phase of our method as $\vy= \dot{\vx} - \mathbf{f}(\vx, \vu, t)$.}
\subsection{Offline Meta-Learning Phase}
In Fig.~\ref{fig:arch}, we \revisiontwo{illustrate the} structure of our proposed solution to learn offline the basis function \revision{$\mphi^\vw(\mathbf{x}, \cE)$} in~\eqref{Eq:Phidef} and to make real-time adjustments
using composite adaptive control (\Cref{sec:adapt_control_def}). 
Our method is divided into two steps. 
First, the robot captures relevant ground information during offline data collection, followed by training a \ac{dnn} with terrain and state information to approximate the residual dynamics $\vy$.
Second, the trained model is deployed onboard the robot and updated online to compensate for the residual dynamics not captured in offline training.

\subsubsection{Dataset}
To learn the basis function \revision{$\mphi^\vw$}  in~\eqref{Eq:Phidef}, we collect a dataset of \revisiontwo{the} robot \revision{operating} on a diverse set of terrains.
The dataset includes paired ground images from the onboard camera and state information measured using onboard sensors.
The images are processed through a VFM, resulting in the representation $\cE$\revision{,}
as discussed in \Cref{sec:model}.

This dataset contains $N \in \N$ trajectories.
Each trajectory is an uninterrupted driving session on the order of a few minutes.
Therefore, a single trajectory may contain significant dynamics-altering terrain transitions, such as between grass and concrete,
but it will not contain dramatic transitions such as from a desert to a forest, or from midday to night.
For notational simplicity only, we assume all trajectories have equal length $\ell \in \N$.
Let $\vx^n_t, \vu^n_t, \cE^n_t, \vy^n_t$ respectively denote the $t\th$ state, input, VFM \revision{representation}, and residual dynamics derivative of the $n\th$ trajectory.

\begin{figure}[t]
   \centering
\includegraphics[width=0.8\columnwidth]{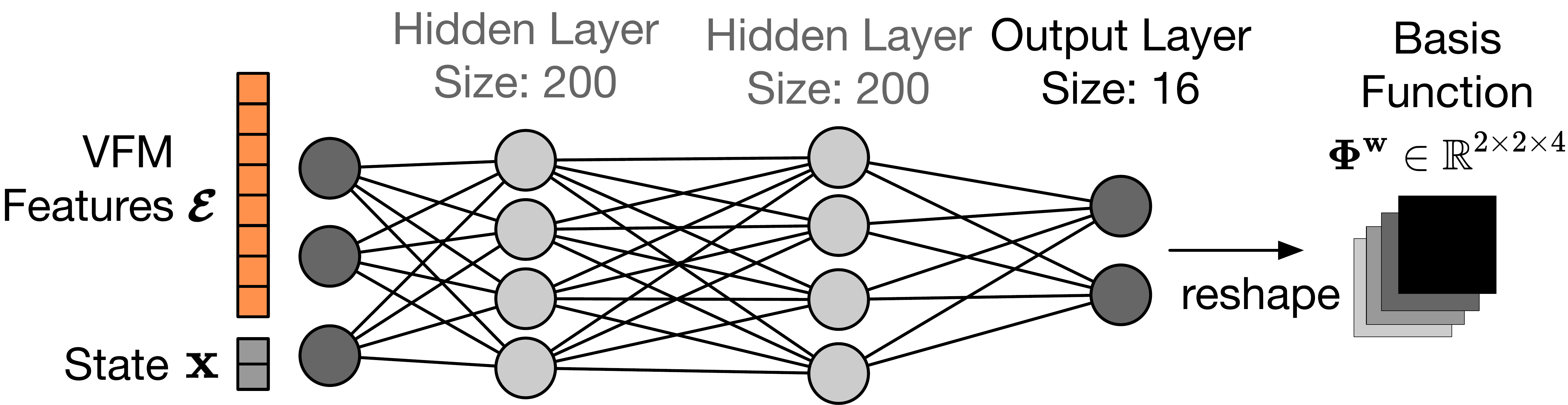}
    \caption{The structure of the DNN used for the basis function $\mphi^\vw$ in the controller synthesis from~\eqref{eq:main_controller},~\eqref{eq:adapt_law} applied to a tracked vehicle.}\label{fig:dnn_architecture}
\end{figure}

\subsubsection{Model Architecture}
In designing the parameterized function class for \revision{the basis function} $\mphi^\vw$, we prioritize simplicity and efficiency to enable \revision{fast} inference for real-time control.
Therefore, we select a fully connected \ac{dnn} with two hidden layers (Fig.~\ref{fig:dnn_architecture})\revisiontwo{, which takes as input }both the robot's state and visual features from the \ac{vfm}.
We employ layer-wise spectral normalization to constrain the Lipschitz constant of the \ac{dnn}.
Spectral normalization is crucial for ensuring smooth control outputs and \revisiontwo{limiting} pathological behavior outside the training domain~\cite{shi2019neural}.
Details of spectral normalization are given in \Cref{sec:training-algo}.

\algblock{Input}{EndInput}
\algnotext{EndInput}
\begin{algorithm}[t]
    \caption{Offline Meta-Learning with Continuous Trajectories}
    \label{alg:training-v2}
    \begin{algorithmic}[1]
        \Input
            \State Dataset of $N$ trajectories of length $\ell$,
            window length distribution $L$,
            regularization target $\vtheta_{\mathrm{r}}$ and weight $\lambda_\mathrm{r}$,
            minibatch size $K$,
            initial DNN weights $\vw$.
        \EndInput
        \State \textbf{Output:} Final optimized DNN weights $\vw $ \revision{ of $\mphi$.}
        \While{not converged}
            \State Sample (with replacement) size-$K$ minibatches of:
                \State \quad - trajectory indices $\{n_k\}_{k=1}^K$,
                \State \quad - window lengths  $\{\window_k\}_{k=1}^K$,
                \State \quad - start times $\{s_k\}_{k=1}^K$.
            \State Solve~\eqref{eq:theta-for-window} for each $\theta^\star_{n_k, \window_k, s_k}$ in the minibatch \label{algo:least_squares}
            \State \quad (in closed form, allowing $J_{n_k, \window_k, s_k}$ gradient flow\footnotemark). 
            \State \label{alg_step:optimization} Take optimizer step on $\vw$ w.r.t minibatch cost \eqref{eq:training-objective}
                \[
                    \textstyle
                    \sum_{k=1}^K 
                    J_{n_k, \hat \ell_k, s_k}(\vw).
                \]
            \State Spectral Normalizaton:
                $\mW_i \gets \frac{\mW_i}{\|\mW_i\|}$ for all $i \in [d]$.
                \label{line:specnorm}
        \EndWhile
    \end{algorithmic}
\end{algorithm}

\subsubsection{Optimization}
\label{sec:training-algo}
Our method is built around the assumption that two terrains with similar visual features $\cE$ will usually, but not always, induce similar dynamics.
We account for this \revisiontwo{observation} with a meta-learning method that allows the linear part $\vtheta$ to vary over the training data while the feature mapping weights $\vw$ remain fixed.
In particular, we assume that the linear part $\vtheta$ is slowly time-varying within a single trajectory in the training data but \revisiontwo{$\vtheta$} may change arbitrarily much between two trajectories.
The slowly time-varying assumption implies that within a sufficiently short window into a full trajectory, $\vtheta$ is approximately constant.
Therefore, we optimize \revision{the weights} $\vw$ \revision{of the basis function} for data-fitting accuracy when the best-fit constant $\vtheta$ is computed for random short windows into the trajectories.
Due to our linear adaptation model structure, 
we observe that 
for a particular trajectory index $n \in [1\dots N]$,
window length $\window \in [1\dots \ell]$, and starting timestep $s \in [1\dots\ell - \window + 1]$, the best-fit value
\revision{
\begin{equation}
\begin{aligned}\label{eq:theta-for-window}
    &
    \argmin_{\vtheta}
    \sum_{t = s}^{s + \window - 1}
    \big\|
    \vy^n_t - (\mphi^\vw\vtheta ) \vu^n_t
    \big\|_2^2 \notag 
     + \lambda_r \| \vtheta - \vtheta_r \|_2^2, \\
    & = \sum_{t = s}^{s + \window - 1} \left( (\mphi^\vw \vtheta)^\top (\mphi^\vw \vtheta) + \lambda_r \mathbf{I}_{n_\theta} \right)^{-1}  \left( (\mphi^\vw \vtheta)^\top \vy^n_t + \lambda_r \vtheta_r \right) \\
    & \eqqcolon \vtheta^\star_{n, \window, s}(\vw),
\end{aligned}
\end{equation}}
is the solution to an $L_2$-regularized linear least-squares problem and is a closed-form, continuous function of the feature mapping parameter $\vw$, \revision{ where $\mphi^\vw$ is shorthand for $\mphi^\vw(\vx^n_t, \cE^n_t)$, $\lambda_r \in \mathbb{R}$ is the regularization parameter, and $\vtheta_r \in \mathbb{R}^{n_\theta}$ is the regularization target, chosen arbitrary.} 
 \revision{The $\Vert \vtheta - \vtheta_r \Vert_2^2$ regularization term ensures that the closed-form solution is unique.}
We can now define our overall optimization objective.
Let $L$ denote a distribution over trajectory window lengths: $L \in \simplex{[1, \ell]}$.
We minimize the meta-objective
\begin{align}
\label{eq:training-objective}
    J(\vw) = &
    \E_{n, \window, s } \left[
        \sum_{t = s}^{s + \window - 1}
        \left\|
            \vy^n_t - \left(\mphi^\vw(\vx^n_t, \cE^n_t)
            \vtheta^\star_{n, \window, s}(\vw)\right)
            \vu^n_t
        \right\|_2^2
    \right] \notag \\
    \coloneqq &
    \E_{n, \window, s } \left[
        J_{n, \window, s}(\vw)
    \right],
\end{align}
where the expectation is shorthand for
$n \sim \uniform[1, N]$,
$\window \sim L$,
and
$s \sim \uniform[1, \ell - \window + 1]$.
\revisiontwo{We incorporate} the closed-form computation of \revision{the best-fit linear \revisiontwo{component}} $\vtheta^\star_{n, \window, s}$ in the computational graph of our optimization \revision{(meaning we take gradients through the least squares solution in Line~\ref{algo:least_squares})},
as opposed to treating the trajectories of $\vtheta$ as optimization variables. \revisiontwo{Given this,}
we obtain a simpler algorithm. 

\begin{algorithm}[t]
    \caption{Rapid \revision{Terrain-Informed} Online Adaptation for Model Mismatch and Tracking Error}
    \label{alg:adaptation_at_execution_time}
	\begin{algorithmic}[1]
    \Input
        \State Optimized feature mapping $\mphi$ from Algorithm \ref{alg:training-v2},
         importance weight for prediction $\mathbf{R}$,
         damping constant $\lambda$,
         initial adaptive gain $\mgamma_0$,
         reference trajectory $\vx_r$.
    \EndInput
    \State Initialize $\hat{\vtheta}$ \revision{with an user-defined regularization target} $ \vtheta_{\mathrm{r}}$.\label{alg:init_alg2}
    \While{running}
        \State Evaluate the \ac{vfm} \revision{on the input image} and get the features $\cE$. \label{alg:evaluate_Phi}
        \State Get state $\vx$ and evaluate the \revision{\ac{dnn}} $\mphi(\vx, \cE)$. \label{alg:evaluate_dnn} 
        \State Compute the \revision{tracking} error vector $\mathbf{s}$ \revision{(e.g., using~\eqref{eq:s_definition})}.
        \State Compute the control input $\mathbf{u}$ \label{alg:control_input}\revision{(e.g., using \eqref{eq:main_controller})}.
        \State \revision{Compute the dynamics residual $\vy$ (e.g., using~\eqref{eq:dyn_res_tracked_vehicle})}
        \State Compute \revision{the} adaptation parameter derivative $\dot{\hat{\vtheta}}$  \label{alg:adapt_composite} \quad \quad
        \revision{$\label{Eq:theta_hat_adapt_rule}
        \dot{\hat{\vtheta}} = -\lambda\hat{\vtheta} -  \mathrm{predict}(\mgamma, \vy, \mphi, \hat{\vtheta}, \vvu) +  \mathrm{track}( \mgamma, \vs, \mphi, \hat{\vtheta}, \vvu ).$}
        \State Compute \revision{the} adaptation gain derivative $\dot{\mgamma}$ \revision{(e.g., using~\eqref{eq:adapt_law} or~\eqref{eq:adapt_law_new})}.\label{alg:adaptation_gain_derivative}

        \State Integrate with system timestep $\Delta_t$ %
                \[
                \hat{\vtheta} \gets  \hat{\vtheta} + \Delta_t \dot{\hat{\vtheta}},
                \quad
                \mgamma \gets \mgamma + \Delta_t \dot \mgamma.
            \]
    \EndWhile
 \end{algorithmic}
\end{algorithm}

Our offline training procedure is given in \Cref{alg:training-v2}.
It consists of stochastic first-order optimization
on the objective $J(\vw)$
and spectral normalization to enforce the Lipschitz constraint on $\mphi$.
In particular, let $\vw = (\mW_1, \dots, \mW_d, \overline \vw)$, \revision{with $d$ being the number of layers},
where $\mW_1, \dots, \mW_d$ are the dimensionally compatible weight matrices of the \ac{dnn}, such that the product
$\mW_1 \cdots \mW_d$ exists,
and $\overline \vw$ are the remaining bias parameters.
It holds that $\| \mphi^\vw \|_\Lip \leq \| \mW_1 \cdots \mW_d \|$ for neural networks with $1$-Lipschitz nonlinearities.
Therefore, we can enforce that
$\| \mphi^\vw \|_\Lip \leq 1$,
by enforcing that $\| \mW_i \| \leq 1$ for all $i \in [1\dots \revision{d}]$.
This is implemented in \Cref{line:specnorm} of \Cref{alg:training-v2}.
Note that finding less conservative ways to enforce
$\| \mphi^\vw \|_\Lip \leq 1$
for \ac{dnn} is an active area of research.

For the remaining sections, the parameters $\vw$ of \revision{the basis function} $\mphi$ are fixed. Therefore, we drop the superscript and refer to $\mphi^\vw$ as $\mphi$, for simplicity of notation.

\footnotetext{We use the machine learning framework PyTorch that implements the linear least squares solution with gradient flow.}

\subsection{Online \revision{$\vtheta$} Adaptation and \revision{Tracking} Control}\label{sec:adapt_control_def}

This section introduces the online adaptation running onboard the robot
to adapt the linear part $\vtheta$ of the dynamics model.
The algorithm \revision{uses composite adaptation\cite{slotine_book} and} is given in Algorithm \ref{alg:adaptation_at_execution_time}. 
The parameter \revision{vector}~$\vtheta$ is initialized with a \revision{user-defined} regularization target $\vtheta_{\mathrm{r}}$ \revision{(Line~\ref{alg:init_alg2})}.
Then, in each cycle of the main loop, the robot processes the data from its visual sensor through the \ac{vfm} to generate the feature vector $\cE$ \revision{(Line~\ref{alg:evaluate_Phi})}.
The feature mapping $\mphi$ is then evaluated using the robot's current state \revision{$\vx$} and \revision{the feature vector} $\cE$ \revision{(Line~\ref{alg:evaluate_dnn})}. 
\revision{We compute the error vector} $\vs$ between the reference trajectory $\vx_r$ and the actual state $\vx$, \revision{for example, using~\eqref{eq:s_definition}}, \revision{ as well as the control input $
\vvu$ (Line \ref{alg:control_input})}.

 \revision{These previously computed variables are} passed to the \revision{composite} adaptive controller.
In this way, model mismatches and other disturbances not captured during training \revision{(\Cref{alg:training-v2})} can be adapted in real-time.
For each sampled measurement (interaction with the environment), the adaptation parameter vector \revision{$\hat{\vtheta}$} is updated using the composite adaptation rule in 
\revision{Line~\eqref{alg:adapt_composite}, which \revisiontwo{is designed} to decrease both the tracking error and the prediction error.}

\revision{Each term of Line~\eqref{Eq:theta_hat_adapt_rule} provides a specific functionality:} the first term in \revision{Line}~\eqref{Eq:theta_hat_adapt_rule} implements the so-called ``exponential forgetting'' to allow $\vtheta$ to change more rapidly when the best-fit parameters are time-varying.
The second term is gradient descent on the $\mR^{-1}$-weighted squared prediction error with respect to \revision{$\hat{\vtheta}$, where $\mR$ is a positive definite matrix}. \revision{The third term  minimize\revisiontwo{s} the trajectory tracking error}.
\revision{In Line \ref{alg:adaptation_gain_derivative}, we introduce $\mgamma$, which is our adaptation gain,} 
\revision{and its derivative $\dot{\mgamma}$} can be defined from least-squares with exponential forgetting~\cite{slotine_book, shi_adaptive_2020} or reminiscent of a Kalman Filter like in~\cite{O_Connell_2022}.
Thus, this composite adaptation is used to ensure both small tracking errors and low model mismatch.
\revision{The functions `$\mathrm{predict}$' and `$\mathrm{track}$' are defined in~\Cref{sec:control_synthesis}}.
The stability and robustness properties of Algorithm~\ref{alg:adaptation_at_execution_time} are presented in~\Cref{sec:adapt_control_tracked_vehicle}.

\section{System Modelling and Control Synthesis}\label{sec:modeling}

We apply the methods from Algorithm~\ref{alg:training-v2} and~\ref{alg:adaptation_at_execution_time} to a \revision{skid-steering} tracked vehicle (Fig.~\ref{fig:frames_of_reference}) \revision{ (\Cref{sec:tracked_vehicle_model} through \ref{sec:adapt_control_tracked_vehicle}) and to an Ackermann-steering vehicle (\Cref{sec:modelling_ackermann} through \ref{sec:ackermann_controller})}.
\revision{The} tracked vehicle uses skid-steering to maneuver over the ground, with its tracks moving at different speeds depending on the sprocket's angular velocity.
Due to the slip between the sprocket and the tracks and between the tracks and the ground, modeling the full dynamics becomes very complex.
We therefore derive its 3 \ac{dof} dynamics model~\eqref{eq:dynamics_with_xicr_eq} with its corresponding simplified model of the form~\eqref{eq:dynamicsAB}.
To this simplified model, we apply an adaptive controller with learned ground information of the form~\eqref{eq:main_controller} \revision{and}~\eqref{eq:adapt_law}.
\revision{
The car-like vehicle uses the Ackermann steering geometry, which ensures that all wheels turn around the same center, thus minimizing wheel wear.
For this vehicle type, we derive a 3-\ac{dof} dynamics model~\eqref{eq:dynamics_traxxas} using the bicycle model and an adaptive controller using Algorithm~\ref{alg:adaptation_at_execution_time}.
}

\subsection{Tracked Vehicle Dynamics Model}\label{sec:tracked_vehicle_model}

We \revision{define} a fixed reference frame $\mathcal{I}$ and a moving reference frame $\mathcal{B}$ attached to the body of the tracked vehicle, as seen in Fig.~\ref{fig:frames_of_reference}.
\begin{figure}[t]
    \centering
\includegraphics[width=\columnwidth]{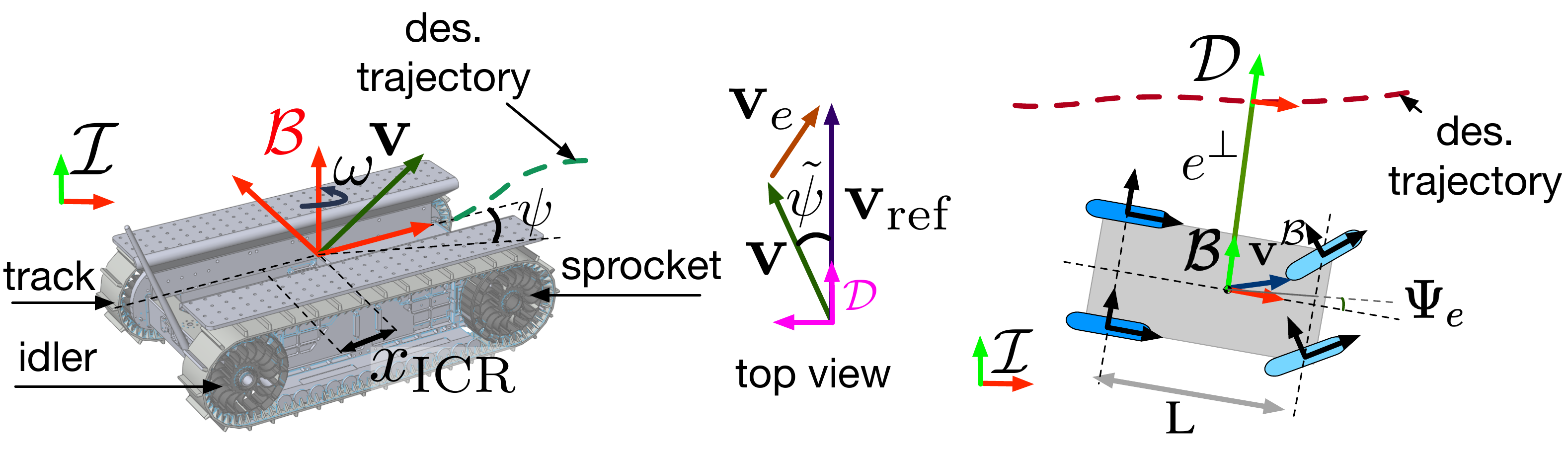}
    \caption{The frames of reference for the tracked vehicle, its corresponding velocities, and the main driving components \revision{(left)}, a velocity vector diagram used for the proof of Theorem~\ref{theorem2} \revision{(middle), and the} car model notations \revision{(right). For both vehicles, we assume the center of mass and the body frame are at the same location.} }
\label{fig:frames_of_reference}
\end{figure}
\revisiontwo{C}onsider the 3-\ac{dof} dynamics model with the generalized coordinates $\mathbf{q}:=[ p_x^\mathcal{I}, p_y^\mathcal{I}, \psi ]\in \mathbb{R}^3$, where $p_x^\mathcal{I}$ and $p_y^\mathcal{I}$ are the inertial positions and $\psi$ is the yaw angle from $\mathcal{I}$ to $\mathcal{B}$, as follows
\begin{equation}
\label{eq:EL}
\mM(\vq) \ddot{\vq} + \mC(\vq, \dot{\vq})\dot{\vq} = \mathbf{B}(\vq) \boldsymbol{\tau}_u + \mathbf{F}_r (\vq, \dot{\vq}),
\end{equation}
where $\mM \in \mathbb{R}^{3 \times 3}$ is the inertia matrix, $\mC(\vq, \dot{\vq}) \in \mathbb{R}^{3 \times 3}$ is the Coriolis and centripetal matrix, $\mathbf{B}(\mathbf{q})  \in \mathbb{R}^{3 \times 2}$ is the control actuation matrix, $\mathbf{F}_r (\vq, \dot{\vq}) \in \mathbb{R}^{3}$ are the dissipative track forces due to surface-to-soil interaction, and $\boldsymbol{\tau}_u  \in \mathbb{R}^{2}$ is the control torque.

Developing a tracking controller for the system modeled using \eqref{eq:EL} is difficult \revisiontwo{due to underactuation}. To address this complexity, previous work~\cite{electronics10020187, BYRNE1995343} introduced a nonholonomic constraint for~\eqref{eq:EL}, which reduces the number of state variables. The following  constrains the ratio of the lateral body velocity $\dot{p}_y^\mathcal{B}$ to the angular velocity $\omega$
\begin{equation}\label{eq:constraint}
\dot{p}_y^\mathcal{B} + x_{\mathrm{ICR}} \omega  = 0,
\end{equation}
where $x_{\mathrm{ICR}}$
is the \ac{icr} \revision{and $\omega = \dot{\psi}$.}
We  embed this constraint into~\eqref{eq:EL}, as follows
\begin{equation} \label{eq:EL_with_constraint}
    \mM(\vq) \ddot{\vq} + \mC(\vq, \dot{\vq})\dot{\vq}  = \mathbf{B}(\mathbf{q}) \boldsymbol{\tau}_u + \mathbf{F}_r (\mathbf{q}, \mathbf{\dot{q}}) + \mathbf{A}(\mathbf{q})^\top \lambda_c,
\end{equation}
with $\lambda_c$ \revision{being} the Lagrange multiplier corresponding to the equality constraint in \eqref{eq:constraint}. 
By assuming $x_\mathrm{ICR}$ \revisiontwo{as} constant, $\mathbf{A}(\mathbf{q})\in\mathbb{R}^{1 \times 3}$ is defined as follows, in which \revision{$\vp$ from \eqref{eq:constraint}} \revision{is} expressed in the $\mathcal{I}$ frame
\begin{equation}\label{eq:nonholonomic_tracked_vehicle}
    \begin{bmatrix}
        - \sin \psi & \cos \psi & x_\mathrm{ICR}
    \end{bmatrix} \cdot
    \begin{bmatrix}
        \dot{p}_x^\mathcal{I} &
        \dot{p}_y^\mathcal{I} &
        \omega
    \end{bmatrix} = \mathbf{A}(\mathbf{q}) \dot{\mathbf{q}} = 0 \revision{.}
\end{equation}
To \revision{remove} the constraint force from  \eqref{eq:EL_with_constraint}, an orthogonal projection operator $\mathbf{S}(\mathbf{q}) \in \mathbb{R}^{3 \times 2}$ is defined, whose columns are in the nullspace of $\mathbf{A}^\top(\mathbf{q})$, and thus $\mathbf{S}(\mathbf{q})^\top \mathbf{A}(\mathbf{q})^\top = \mathbf{0}$ \cite{mistry_operational_2011, aghili_unified_2005}. 
\begin{equation}\label{eq:S_eq}
    \mathbf{S}(\mathbf{q}) =\left[
    \begin{smallmatrix}
        \cos(\psi) & x_{\mathrm{ICR}} \sin(\psi) \\
    \sin(\psi) & -x_{\mathrm{ICR}} \cos(\psi) \\
    0 & 1
    \end{smallmatrix}\right] \revision{.}
\end{equation} %
We \revision{select} this projection operator conveniently to transform the velocities in the $\mathcal{I}$ frame to $\mathbf{v} = [v_x^\mathcal{B}, \omega]^\top$, with $v_x^\mathcal{B}$ being the projection of the inertial velocity onto the body x-forward axis.
The reduced form can be written as~\cite{trajectory_tracking_control}
\begin{equation}\label{eq:dynamics_with_xicr_eq}
\begin{aligned}
& \dot{\mathbf{q}}(t)=\mathbf{S}(\mathbf{q}) \mathbf{v}(t), \\
& \dot{\mathbf{v}}(t)=\widetilde{\mM}^{-1}(\widetilde{\mB}(\mathbf{q})\boldsymbol{\tau}_u  - \widetilde{\mathbf{C}}(\mathbf{q}, \dot{\mathbf{q}}) \vv(t) +\widetilde{\mathbf{F}}_r(\mathbf{q}, \dot{\mathbf{q}})),
\end{aligned}
\end{equation}
with the reduced matrices
$$
\begin{aligned}
\widetilde{\mathbf{M}} & =\mathbf{S}^\top(\mathbf{q}) \mathbf{M S(\mathbf{q})}=\left[\begin{smallmatrix}
m & 0 \\
0 & I_z  + m x_\mathrm{ICR}^2
\end{smallmatrix}\right],  \\  
\widetilde{\mathbf{F}}_r &=\mathbf{S}^\top(\mathbf{q}) \mathbf{F}_r, \quad \widetilde{\mathbf{B}}(\mathbf{q})=\mathbf{S}^\top(\mathbf{q})\mathbf{B}(\mathbf{q}),  \\
\widetilde{\mC}(\vq, \dot{\vq}) &=\mathbf{S}^\top \mathbf{M} \dot{\mathbf{S}}=\left[\begin{smallmatrix}
0 & m x_{\mathrm{ICR}} \omega \\
-m x_\mathrm{ICR} \omega & 0
\end{smallmatrix}\right],
\end{aligned} %
$$
where $m$ is the mass of the robot and $I_z$ is the inertia of the robot about the rotational degree of freedom. 

\subsection{Simplified Vehicle Dynamics Model with Velocity Input}
Due to the limited access to the robot's internal control software, specifically the absence of direct torque command capabilities, we are only able to utilize velocity inputs.
Consequently, we have chosen to simplify the system in~\eqref{eq:dynamics_with_xicr_eq} \revision{with} the velocity modeled as a first-order time delay
\begin{equation}\label{eq:dynamicsAB}
\dot{\mathbf{q}}(t)=\mathbf{S}(\mathbf{q}) \mathbf{v}(t), \quad 
\dot{\mathbf{v}}(t) = \mA \mathbf{v} (t) + \mathbf{B}_\mathrm{n} \mathbf{u}(t),
\end{equation}
where $\mathbf{u} = [u_v, u_\omega]$ are velocity inputs and 
\begin{eqnarray}
\mA = \left[\begin{smallmatrix}
    -\frac{1}{\tau_{v}} & 0 \\
    0 & -\frac{1}{\tau_{\omega}}
    \end{smallmatrix}\right], \quad 
    \mathbf{B}_\mathrm{n} = 
    \left[\begin{smallmatrix}
    \frac{k_1}{\tau_{v}} & 0 \\
    0 & \frac{k_2}{\tau_{\omega}}
    \end{smallmatrix}\right].
\end{eqnarray}
The simplified system is dynamically equivalent to~\eqref{eq:dynamics_with_xicr_eq}. We identify the process gains $k_1$, $k_2$ and the process time constants $\tau_v$, $\tau_\omega$ using system identification on \revision{hardware}.
The robot is symmetric and rotates around the origin, therefore $x_\mathrm{ICR}$ is assumed 0.

\subsection{Adaptive Tracking Controller for \revision{a Tracked Vehicle}}\label{sec:adapt_control_tracked_vehicle}
First, we explain why using a control matrix adaptation is suitable for adapting to longitudinal and rotational slips, as well as the internal dynamics of a tracked vehicle.
Next, we design a composite adaptive controller for the system in~\eqref{eq:dynamicsAB} and prove its exponential convergence to a bounded error ball. Note that our adaptive controller can be applied to any system of the form~\eqref{eq:EL}.

\subsubsection{Motivation for Control Matrix Adaptation}\label{sec:motivation_control_matrix}
The longitudinal slip $\kappa$ is defined~\cite{pacejka2005tire} as
\begin{equation}\label{eq:slip}
\kappa=-\frac{v_x^\mathcal{B}-\Omega_{\mathrm{tr}} r_{\mathrm{tr}}}{v_x^\mathcal{B}},
\end{equation}
where $\Omega_\mathrm{tr}$ is the angular velocity of the tracks, $r_\mathrm{tr}$ is the track wheel radius, and $v_x^\mathcal{B}$ is the projection of the inertial velocity onto the body x-forward
axis.
Let $\Omega_\mathrm{tr} r_\mathrm{tr}$ be our velocity control input $u_v$. Then~\eqref{eq:slip} can be written as $v_x^\mathcal{B} = \frac{1}{1 + \kappa} u_v$. Analyzing the extreme cases, we notice that if $\kappa = 0$ (no longitudinal slip), the velocity of the vehicle will match the velocity input into the tracks. In comparison, if $\kappa \rightarrow \infty$, the forward velocity of the vehicle will tend toward zero. Similar reasoning can be applied to the rotational slip.
Therefore, adapting for a coefficient that multiplies the control input (the track speeds) ensures tracking of the body's forward and angular velocity.

In addition, adapting the control matrix also contributes to compensating for the unknown internal dynamics of the robot, because the velocity control input is the setpoint to an internal proportional-derivative-integral controller, which outputs motor torques to the tracked vehicle. 
Lastly, adapting the control matrix effectively compensates for track degradation, manifested as a slowdown in the sprocket's angular velocity.

\subsubsection{Reference Trajectories}

We define a 2 dimensional \emph{feasible} trajectory characterized by the desired position and velocity $\mathbf{p}_d^\mathcal{I}(t)$, $ \mathbf{v}_d^\mathcal{I}(t)$ in the inertial frame $\mathcal{I}$. The position error is $\mathbf{\tilde{p}}^\mathcal{I} = \mathbf{p}^\mathcal{I} - \mathbf{p}_d^\mathcal{I}(t)$, where $\mathbf{p}^\mathcal{I} = [p_x^\mathcal{I}, p_y^\mathcal{I}]$, and $\psi_{\mathrm{d}}$ is the desired yaw angle.
Let the following reference velocities be
\begin{equation}\label{eq:v_I_error}
\mathbf{v}_{\mathrm{ref}}^\mathcal{I} = \mathbf{v}_d^\mathcal{I} - \mathbf{K}_p \tilde{\mathbf{p}}^\mathcal{I}, \quad v_{\mathrm{ref, x}} =  \begin{bmatrix}
    \cos (\psi)  \\
    \sin (\psi)
\end{bmatrix}  \cdot  \mathbf{v}_{\mathrm{ref}}^\mathcal{I},
\end{equation}
\begin{equation}\label{eq:kinematic_angle}
\begin{aligned}
\omega_{\mathrm{ref}} &= \dot{\psi}_\mathrm{ref} - k_{\mathrm{\psi}}(\psi - \psi_\mathrm{ref}),
\end{aligned}
\end{equation}
where \revision{the reference angle is given as}
\begin{equation}\label{eq:psi_ref}
\psi_\mathrm{ref}  = \left\{\begin{array}{ll}
\arctan\left(\frac{\mathbf{v}_{\mathrm{ ref,y}}^\mathcal{I}}{\mathbf{v}_{\mathrm{ref, x}}^\mathcal{I}} \right), & \text{if } \Vert \mathbf{v}_{\mathrm{ ref}}^\mathcal{I} \Vert_2^2 > v_\mathrm{\epsilon} \\
\psi_\mathrm{d},  &\text{otherwise.}
\end{array}\right.
\end{equation}
Note that the reference trajectory is not fully pre-planned; it includes feedback terms that are only defined during the execution of the trajectory.
Both $\mathbf{K}_p  \in \mathbb{R}^{2 \times 2}$ and $k_{\psi} \in \mathbb{R}$ are positive gains, with $\mathbf{K}_p=\mathrm{diag}(k_{px}, k_{py})$, and $v_\epsilon$ is a small velocity constant used to ensure the robot can track time-varying position trajectories, as well as turn in place.
We define $\mathbf{v}_{\mathrm{ref}} = [v_{\mathrm{ref, x}}, \omega_\mathrm{ref}]^\top$, which is our reference trajectory \revision{further used in the control synthesis}.

\subsubsection{Controller Synthesis}
\label{sec:control_synthesis}
We design a composite adaptive controller $\mathbf{u}(t)$ and show that this composite tracking and adaptation error exponentially converge\revisiontwo{s} to a bounded error ball. 
\revision{First, we start by defining the}  tracking error variable $\mathbf{s}$ as 
\begin{equation}\label{eq:s_definition}
\mathbf{s}=\vv - \vv_\text{ref}=[v_x^\mathcal{B}-v_{\mathrm{ref, x}}, \ \omega-\omega_\mathrm{ref}]^\top.
\end{equation}
\revision{We then derive the tracking controller for the system in~\eqref{eq:dynamicsAB}}
\begin{equation}\label{eq:main_controller}
    \vvu = - (\mathbf{B}_\mathrm{n} + \mphi \hat{\vtheta})^{-1} [\mathbf{K} \mathbf{s} + \mA \vv_\mathrm{ref} - \dot{\vv}_\mathrm{ref}],
\end{equation}
where $\mK \in \mathbb{R}^{2 \times 2}$ is a positive gain matrix, with $\mK = \mathrm{diag}(k_{dx}, k_{dw})$, \revision{
$\mphi \in \mathbb{R}^{2 \times 2 \times n_\theta}$ is the output of the \ac{dnn} basis function \revisiontwo{(\Cref{fig:dnn_architecture})} evaluated with the feature vector $\cE$ and state $\vv$, 
and $\hat{\vtheta} \in \mathbb{R}^{n_\theta}$ is the estimated parameter vector of the true parameter vector $\vtheta$.
Recall from Sec.~\ref{sec:model} that the learned basis function $\mphi$ and the true adaptation parameters $\vtheta$
were introduced to model the disturbance $\mathbf{d} \approx (\mphi \boldsymbol{\theta}) \mathbf{u} = \sum_{i=1}^{n_\theta} \theta_i \mphi \vu$. 
For our model of the skid-steer vehicle, we chose $n_\theta=4$, matching the number of terms in our control matrix $\mB_\mathrm{n}$. Choosing $n_\theta$ too large can introduce redundant parameters and 
choosing $n_\theta$ too small could make the function class insufficiently expressive. 
}

\begin{theorem}\label{theorem1}
By applying the controller in \eqref{eq:main_controller} to the dynamics that evolve according to~\eqref{eq:dynamicsAB}, with the composite adaptation law, for each $i \in [1, n_\theta]$
\begin{align}\label{eq:adapt_law}
    \dhai &= - \lambda \hai - \underbrace{\gamma_i \vvu^\top \mphi_i^\top \mR^{-1} \left(\sum_{i=1}^{n_\theta} \mphi_i \vvu \hai - \vy\right)}_{\revision{\mathrm{predict}}} + \underbrace{\gamma_i \vs^\top \mphi_i \vvu}_{\revision{{\mathrm{track}}}}, \notag \\
   \dot{\gamma}_i &= - 2 \lambda \gamma_i + q_i + \gamma_i \vvu^\top \mphi_i^\top \mR^{-1}\mphi_i \vvu \gamma_i,
\end{align}
where $\gamma_i > 0,\ q_i > 0$, for each  $i \in [1, n_\theta]$,
then the tracking errors $\mathbf{s}$ and the parameter error $\tilde{\vtheta}$ will exponentially converge to a bounded error ball.
\end{theorem}

\begin{proof}
Defining the true control matrix as $\mathbf{B} := \mathbf{B}_\mathrm{n} + \mphi \boldsymbol{\theta}$, we obtain the following closed loop system \revision{using} \eqref{eq:dynamicsAB}
$$
\dot\vv = \mA \vv - (\mathbf{B}_\mathrm{n} + \mphi \boldsymbol{\theta}) (\mathbf{B}_\mathrm{n} + \mphi \hat{\boldsymbol{\theta}})^{-1} [\mathbf{K} \mathbf{s}+ \mA \vv_\mathrm{ref} - \dot\vv_\mathrm{ref}] + \vdelta,
$$
\revision{where $\boldsymbol{\delta}$ is a representation error, previously introduced in \eqref{Eq:Phidef_original}.}
Let $\tilde{\boldsymbol{\theta}} = \hat{\boldsymbol{\theta}} - \boldsymbol{\theta}$ be the error adaptation vector. Further, using the composite variable $\vs$ in~\eqref{eq:s_definition}, the closed-loop system becomes
$$
\mathbf{\dot{s}}  = \mA \mathbf{s} - \mathbf{K}\mathbf{s} - (\mphi\tilde{\boldsymbol{\theta}})\mathbf{u} + \vdelta
              = \mA\mathbf{s} - \mathbf{K}\mathbf{s} -   \sum_{i=1}^{n_\theta} \mphi_i \mathbf{u} \tilde{\theta}_i + \vdelta.
$$
For the prediction term in \eqref{eq:adapt_law}, 
we \revision{compute} the dynamics residual derivative $\mathbf{y}$ determined for the bounded and adversarial noise $\bar{\boldsymbol{\epsilon}}$ as
\begin{equation}\label{eq:dyn_res_tracked_vehicle}
\vy=\mathrm{LPF(s)}\dot{\revision{\vv}} - \mathbf{f}(\revision{\vv}, \vu, t) =(\mphi  \boldsymbol{\theta})\vvu+\bar{\boldsymbol{\epsilon}},
\end{equation}
where premultiplying the noisy measurement $\dot{\revision{\vv}}$  by $\mathrm{LPF}(s)$ with the Laplace transform variable $s$ indicates low-pass filtering.
Using the Lyapunov function
$
   \mathcal{V} = \vs^\top \vs + \sum^{n_\theta}_{i = 1} \tilde{\theta}\revision{_i} \gamma_i^{-1} \tilde{\theta}\revision{_i},
$
we compute its derivative as follows
\begin{equation}
\begin{aligned}
       \dot{\mathcal{V}} &= 2 \vs^\top \dot{\vs} + 2 \sum_{i=1}^{n_\theta} \tai \gamma_i^{-1} \dtai + \sum_{i=1}^{n_\theta} \tai \frac{d}{dt}\left(\gamma_i^{-1}\right) \tai\\
    &= -2\vs^\top\Big[(\mK - \mA)\vs + \sum_{i=1}^{n_\theta}\mphi_i  \vvu \tai \Big] \\ &+2 \sum_{i=1}^{n_\theta}\gamma_i^{-1}\tai (\gamma_i \vs^\top\mphi_i \vvu - \gamma_i\vvu^\top \mphi_i^\top\mR^{-1} \sum_{j=1}^{n_\theta} \mphi_j \mathbf{u} \tilde{\theta}_j  - \lambda\tai)\\ 
        &+ \sum_{i=1}^{n_\theta} \tai(2\gamma_i^{-1}\lambda - \gamma^{-1}_i q_i \gamma^{-1}_i - \vvu^\top \mphi_i^\top \mR^{-1}\mphi_i \vvu)\tai\\
        &+ 2\underbrace{
        \left(\vs^\top\vdelta +
        \sum_{i=1}^{n_\theta} \tai \left(\vvu^\top \mphi_i^\top \mR^{-1}\bar{\epsilon} -  \gamma_i^{-1} \lambda\theta_i - \gamma_i^{-1} \dot{\theta}_i\right)\right)}_{\text{error terms}}\nonumber.\\
\end{aligned}
\end{equation}
After further manipulation, the time derivative of the Lyapunov function becomes
\begin{equation}\label{eq:derivative_Lyapunov}
\begin{aligned}
    \dot{\mathcal{V}} &= -2\vs^\top(\mK - \mA) \vs - \sum_{i=1}^{n_\theta} \tai (q_i \gamma^{-2}_i + \vvu^\top \mphi_i^\top \mR^{-1}\mphi_i \vvu)\tai  \\
  & -2 \left(\sum_{i=1}^{n_\theta} \tilde{\theta}_i \vvu^\top \mphi_i^\top\right)\mR^{-1}\left( \sum_{j=1}^{n_\theta} \mphi_j \vvu \tilde{\theta}_j\right) +\text{error terms}. \\
\end{aligned}
\end{equation}
\revisiontwo{Next, we will bound the terms in~\eqref{eq:derivative_Lyapunov} as follows.} There exists $\alpha \in \mathbb{R}_+$ such that 
\begin{equation}
    \begin{aligned}
        -2(\mK - \mA) &\preceq -2\alpha \mathbf{I}, \\
        \revisiontwo{-} \left( q_i \gamma^{-2}_i + \vvu^\top \mphi_i^\top \mR^{-1}\mphi_i \vvu  \right) & \le -2\alpha \gamma_i^{-1}, \quad \forall i \in [1, n_\theta].
    \end{aligned}
\end{equation}
We assume that $\Vert \vdelta \Vert $, $\Vert\bar{\boldsymbol{\epsilon}}\Vert$, and $\dot{\theta}_i$ are small and bounded, and that the true value $\theta_i$ is bounded.
Furthermore, the \ac{dnn} $\mphi_i$ is bounded since we use spectral normalization and the input domain is bounded. 
We then define an upper bound for the error terms as
\begin{equation}
    \begin{aligned}
         \bar{d}
         = \sup_t \left(\abs{\vdelta} + 
          \left|\sum_{i=1}^{n_\theta} \left(\vvu^\top \mphi_i^\top \mR^{-1}\bar{\boldsymbol{\epsilon}} -  \gamma_i^{-1} \lambda\theta_i - \gamma_i^{-1} \dot{\theta}_i \right)\right|\right),
    \end{aligned}
\end{equation}
\revision{Note that this is a conservative estimate (the worst-case disturbance of all future time $t$), and hence can be made smaller using a shorter time range. Furthermore, even for a relatively large value of $\mphi$, $\bar{d}$ can be made small using a larger value of $\mR$ and a smaller value of $\bar{\boldsymbol{\epsilon}}$.} We define the matrix $\mathcal{M}$, for $i \in [1, n_\theta]$ 
\begin{equation}
    \begin{aligned}
        \mathcal{M} = \begin{bmatrix}
            \mathbf{I} & \mathbf{0} \\
            \mathbf{0} & \mathrm{diag}(\gamma_i^{-1})
        \end{bmatrix}. %
    \end{aligned}
\end{equation}
By applying the Comparison Lemma~\cite{khalil_nonlinear_2002} \revision{and using a contraction theory like argument~\cite{lohmiller1998contraction,tsukamoto2021contraction}}, we can then prove the tracking error and adaptation parameters \revision{error} exponentially converge to the bounded error ball
\begin{equation}\label{eq:error_bound_s}
    \begin{aligned}
        \lim_{t \to \infty} \abs*{\begin{bmatrix}
            \vs \\
            \tai
        \end{bmatrix}} \le \frac{\bar{d}}{\alpha\lambda_{\min}(\mathcal{M})} :=\revision{\bar{b}},
    \end{aligned}
\end{equation}
where $\lambda_{\min}$ is the minimum eigenvalue of a square matrix.
It follows from \cite{khalil_nonlinear_2002,chung2013phase} that the input-to-state stability (ISS) and bounded input and bounded output (BIBO) stability in the sense of finite-gain $\mathcal{L}_p$~\cite{khalil_nonlinear_2002} is proven for $\bar{b}\in \mathcal{L}_{pe}$, resulting in its bounded output $\mathbf{s}, \boldsymbol{\tilde{\theta}}\in \mathcal{L}_{pe}$, where the $\mathcal{L}_p$ norm in the extended space $\mathcal{L}_{p e}, p \in[1, \infty]$ is 
$$
\begin{aligned}
& \left\|(\mathbf{u})_\tau\right\|_{\mathcal{L}_p}=\left(\int_0^\tau\|\mathbf{u}(t)\|^p d t\right)^{1 / p}<\infty, \quad p \in[1, \infty) \\
& \left\|(\mathbf{u})_\tau\right\|_{\mathcal{L}_{\infty}}=\sup _{t \geq 0}\left\|(\mathbf{u}(t))_\tau\right\|<\infty
\end{aligned}
$$
and $(\mathbf{u}(t))_\tau$ is a truncation of $\mathbf{u}(t)$, i.e., $(\mathbf{u}(t))_\tau=\mathbf{0}$ for $t \geq \tau$, $\tau \in[0, \infty)$ while $(\mathbf{u}(t))_\tau=\mathbf{u}(t)$ for $0 \leq t \leq \tau$. 
\end{proof}
\revision{
The exponential convergence proof in~\Cref{theorem1} shows that the online algorithm (\Cref{alg:adaptation_at_execution_time}) will drive $\boldsymbol{\hat{\theta}}$ to a value within a bounded error ball of the offline least-squares solution used in \revisiontwo{the} meta-learning algorithm (\Cref{alg:training-v2}) for a sufficiently long window of data.}
In contrast with \cite{slotine_book, O_Connell_2022}, \eqref{eq:main_controller} and~\eqref{eq:adapt_law} admits adaptation through the $\mathbf{B}$ control influence matrix. For stability purposes, under the assumption of a diagonal $\mgamma$, the adaptation law equation resembles the \revisiontwo{Riccati} equation of the $\mathcal{H}_\infty$ filtering \cite{doi:https://doi.org/10.1002/0470045345.ch11}. This tends to increase the adaptation gain, making it more responsive to measurements.

The parameters of the adaptation law~\eqref{eq:adapt_law} are $\boldsymbol{\Gamma}=\mathrm{diag}(\gamma_1,\dots, \gamma_{n_{\theta}})$, $\mathbf{R}$, $\lambda$, and $\mathbf{Q} = \mathrm{diag}(q_i)$.
$\boldsymbol{\Gamma}$ is a positive definite matrix that influences the convergence rate of the estimator, and a sufficiently large initial $\boldsymbol{\Gamma}_0$ should be chosen to obtain a suitable convergence rate. $\mathbf{Q}$ is a positive definite gain added to the gain update law,
$\lambda$ is a damping factor, \revision{and} $\mathbf{R}$ is a gain added to the prediction component of the adaptation law. Without this gain, the prediction term and the tracking error-based term could not be tuned separately.

Next, we assume the adaptation gain \revisiontwo{$\Gamma$} in~\eqref{eq:adapt_law} has cross terms. Under this more general setting, we prove the exponential convergence of both $\tilde{\vtheta}$ and $\vs$ to a bounded error ball. 
\begin{proposition} \label{col:new_adapt_law_controller}
By applying the controller in~\eqref{eq:main_controller} to the dynamics in~\eqref{eq:dynamicsAB}, with the composite adaptation law
\begin{subequations} 
\label{eq:adapt_law_new}
\begin{eqnarray}
\dot{\hat{\vtheta}} = -\lambda\hat{\vtheta} - \underbrace{\mgamma \mH^\top \mR^{-1} (\mH \hat{\vtheta} - \vy)}_{\revision{\mathrm{predict}}} + \underbrace{\mgamma\mH^\top \vs}_{\revision{\mathrm{track}}}, \label{eq:adapt_law_new_a}\\
\dot{\mgamma} = -2\lambda \mgamma + \mQ - \mgamma \mH^\top \mR^{-1} \mH \mgamma\label{eq:adapt_law_new_b}, 
\end{eqnarray}
\end{subequations}
where $\lambda > 0$, $\mgamma \in \mathbb{R}^{n_\theta \times n_\theta}$, $\mQ ^{n_\theta \times n_\theta}$ and $\mR\in\mathbb{R}^{2 \times 2}$ are positive definite matrices and \revision{$\mathbf{H}\in \mathbb{R}^{n \times n_\theta}$}, the tracking errors $\mathbf{s}$ and the parameter error $\tilde{\vtheta}$ exponentially converge to a bounded error ball defined in~\cite{O_Connell_2022}.
\end{proposition}
\begin{proof}
\revision{We define the matrix $\mathbf{H} = \left[\mathbf{h}_1 \ldots \mathbf{h}_{n_\theta}\right]\in \mathbb{R}^{n \times n_\theta}$, where the columns $\mathbf{h}_i=\boldsymbol{\Phi}_i \mathbf{u}$, for each $i \in\left[1, n_\theta\right]$.
}
By observing that \revision{the disturbance in~\eqref{Eq:Phidef} can be defined as $\mathbf{d} \approx $}$\sum_{i=1}^{n_\theta} \mphi_i\theta_i\mathbf{u} := \mH \vtheta$, the proof of exponential convergence \revision{for additive disturbance adaptation} from~\cite{O_Connell_2022} can be directly applicable \revision{for the multiplicative disturbance adaptation. }
\end{proof}
Note that the \revisiontwo{exponential} convergence proof for \revisiontwo{Proposition}~\ref{col:new_adapt_law_controller} using Lyapunov theory holds for both when the last term of the gain adaptation \revisiontwo{law}~\eqref{eq:adapt_law_new_b} is positive and when the last term is negative.
A negative sign makes the update law~\eqref{eq:adapt_law_new_b} resemble the covariance update law of the Kalman filter. \revision{However, using a positive sign will make the closed-loop system converge faster \revisiontwo{for the same constants}. Our controller in \Cref{theorem1} behaves similar to the second case with the assumption that $\mgamma$ is diagonal.}

Lastly, for completeness, we show exponential convergence to a bounded error ball for the position and the attitude error.

\begin{theorem}\label{theorem2}
By Theorem 1, $\vs$ converges to a bounded error ball~\eqref{eq:error_bound_s} defined as $\bar{b}$. Therefore, we hierarchically show that $\psi \rightarrow \psi_\mathrm{ref}$ and  $\vp \rightarrow \vp_d$ exponentially fast to a bounded error ball for bounded reference velocity.
\end{theorem}
\begin{proof}
We define the error $\tilde{\psi} = \psi - \psi_{\mathrm{ref}}$. Using \eqref{eq:kinematic_angle}, we obtain
$
\dot{\tilde{\psi}} + k_{\psi} \tilde{\psi} \leq \bar{b},
$
and with the Comparison Lemma, we prove that the error $\tilde{\psi}$ converges to the bounded error ball $\frac{\bar{b}}{k_{\psi}}$.
To give intuition about the following position tracking error proof, we use Fig.~\ref{fig:frames_of_reference}. We define $\vv = \vv_{\mathrm{ref}} + \vv_e$ in vector form, where $\vv_e$ is the velocity error.
We further express these quantities in the reference frame $\mathcal{D}$ and note that, by Theorem~\ref{theorem1}, we have proved the convergence $v_x^\mathcal{B} = v_{\mathrm{ref,x}} + \bar{b}$ as $t\rightarrow \infty$. Therefore, we obtain
$$
\begin{aligned}
\vv_e^\mathcal{D} & = 
- \begin{bmatrix} 
v_{\mathrm{ref, x}}^\mathcal{D} \\
0
\end{bmatrix} 
+ \begin{bmatrix}
\cos(\tilde{\psi}) \\ 
\sin(\tilde{\psi}) 
\end{bmatrix} 
\left(\begin{bmatrix}
\cos(\tilde{\psi}) \\ \sin(\tilde{\psi}) 
\end{bmatrix} \cdot  
\begin{bmatrix}
v_{\mathrm{ref, x}}^\mathcal{D} \\
0 
\end{bmatrix} + \bar{b}\right)  \\
& =  v_{\mathrm{ref, x}}^\mathcal{D} \begin{bmatrix}
\cos^2(\tilde{\psi}) - 1 \\
\sin(\tilde{\psi}) \cos(\tilde{\psi})
\end{bmatrix} + \bar{b} \begin{bmatrix}
\cos(\tilde{\psi}) \\ \sin(\tilde{\psi}) 
\end{bmatrix}. 
\end{aligned}
$$
We compute and bound the norm, as follows
\begin{equation}\label{eq:bound_ve}
\begin{aligned}
\Vert \vv_e^\mathcal{D}  \Vert_2
& \leq  v_{\mathrm{max}}\left|\sin\left(\frac{\bar{b}}{k_\psi}\right)\right| \sqrt{2}  + \bar{b},
\end{aligned}
\end{equation}
where $v_{\mathrm{max}}$ is our assumption for the existence of an upper bound for the reference velocity. From~\eqref{eq:v_I_error} and~\eqref{eq:bound_ve}, it is straightforward to see that the position error is also bounded \revisiontwo{using the Comparison Lemma}.
\end{proof}

\revision{
\subsection{Ackermann Steering Vehicle Dynamics Model}\label{sec:modelling_ackermann}
We define a fixed reference frame $\mathcal{I}$, a moving reference frame $\mathcal{B}$ attached to the center of mass of the car, and a desired frame $\mathcal{D}$ attached to the desired trajectory as seen in Fig.~\ref{fig:frames_of_reference}.
Similar to \eqref{eq:nonholonomic_tracked_vehicle}, a non-holonomic constraint holds: 
$ 
    \begin{bmatrix}
        \sin \psi & -\cos \psi & 0
    \end{bmatrix} \cdot
    \begin{bmatrix}
        \dot{p}_x^\mathcal{I} &
        \dot{p}_y^\mathcal{I} &
       0
    \end{bmatrix},
$ 
where $\dot{p}_x^\mathcal{I}$ and $\dot{p}_y^\mathcal{I} $ are the velocities in the inertial frame $\mathcal{I}$ and $\psi$ is the yaw angle from $\mathcal{B}$ to $\mathcal{I}$.
For the tracked vehicle discussed in~\Cref{sec:tracked_vehicle_model}, the instantaneous center of rotation $x_{\mathrm{ICR}}$ is assumed to be 0 with $\dot{p}_y^\mathcal{B}=0$ in~\eqref{eq:constraint} because the tracked vehicle is not designed for highly aggressive maneuvers.}

\revision{
A car, on the other hand, can be drifting, and thus the side velocity plays a much more important role, which is considered in our control design.
Let
$v_x^\mathcal{B}$ and $v_y^\mathcal{B}$ be the linear velocities in the body frame and $\omega$ the angular velocity around the vertical z-axis of the $\mathcal{B}$ frame.
The dynamic model can be expressed as~\cite{car_darpa}
\begin{equation}\label{eq:dynamics_traxxas}
\begin{gathered}
m(\dot{v}_x^\mathcal{B} - \omega v_y^\mathcal{B}) =  F_{x r} +F_{x f} \cos (u_\delta) - F_{y f} \sin (u_\delta), \\
m(\dot{v}_y^\mathcal{B} + \omega v_x^\mathcal{B}) = F_{y r}+F_{x f} \sin (u_\delta)+F_{y f} \cos (u_\delta), \\
I_z \dot{\omega}=  \frac{L}{2} F_{x f} \sin (u_\delta) + \frac{L}{2} F_{y f} \cos (u_\delta) - \frac{L}{2} F_{yr},
\end{gathered}
\end{equation}
where $L$ is the wheelbase length, $F_{xf}$ and 
$F_{xr}$ are the front and rear tire forward forces, $m$ is the vehicle mass, $I_z$ is the vehicle inertia about the vertical axis intersecting the center of mass, and the lateral forces are
$F_{y f} \approx C_y \alpha_f$, 
$F_{y r}  \approx C_y \alpha_r$, where $C_y$ is the tire cornering stiffness
\cite{wong2001theory}, and $\alpha_r$ and $\alpha_f$ are two tire slip angles, defined as in~\cite{car_darpa}.
The tire cornering stiffness coefficient is terrain- and wheel-dependent. Either an accurate estimate or online adaptation is necessary \revisiontwo{when designing a tracking controller}.
Note that a more slippery ground has a lower $C_y$.
}

\revision{
We decouple the controller for the longitudinal velocity from the controller for the lateral and angular velocity and apply our \namepaper~algorithm to the lateral and angular motion.
Note that the forward velocity dynamics is nonlinear. \revisiontwo{T}herefore, for simplicity, similar to the tracked vehicle, we model the forward velocity as a first-order time delay system $\dot{v}_x^\mathcal{B} = - \tau_v^{-1} (v_x^\mathcal{B} - u_v)$, with the time constant $\tau_v$ identified through system identification. \revisiontwo{We then} design an exponentially stabilizing PD tracking controller for this linear system.
For the lateral and angular motion, we linearize~\eqref{eq:dynamics_traxxas} around a zero steering angle and assume a small tire slip angle.
The resulting linear time-varying system dynamics in the $\mathcal{B}$ frame is written by using
$\vx := [v_y^\mathcal{B}, \omega]$ and the disturbance model~\eqref{Eq:Phidef}
\begin{equation}\label{eq:simplified_dynamics_traxxas}
\dot{\vx} = \mA(v_x^\mathcal{B}(t))  \vx + \mB_\mathrm{n} u_\delta + \left(\mphi(\vx, \cE) \vtheta\right) u_\delta,
\end{equation}
where 
$$
\mA(v_x^\mathcal{B}(t)) = \left[\begin{smallmatrix}
-\frac{2C_y}{m v_x^\mathcal{B}} & -v_x^\mathcal{B}\\
0 & -\frac{L^2 C_{y}}{2 v_x^\mathcal{B} I_z} 
\end{smallmatrix}\right], 
\mB_\mathrm{n} = \left[\begin{smallmatrix}
\frac{C_y}{m}\\
\frac{LC_y}{2I_z}
\end{smallmatrix}\right], 
$$
and $
 \mphi(\vx, \cE) = \begin{bmatrix}
\boldsymbol{\phi}_1 &
\boldsymbol{\phi}_2
\end{bmatrix} ^ \top
$,
with the estimated \revisiontwo{component} $\hat{\vtheta}$ adapted online.
The definition of $\mphi(\vx, \cE)\in\mathbb{R}^{2 \times n_\theta}$ and $\vtheta \in \mathbb{R}^{n_\theta}$ \revisiontwo{are} the same as their definitions for the tracked vehicle. The adaptation component account\revisiontwo{s} for model mismatches as well as for the linearization errors in~\eqref{eq:dynamics_traxxas}.
\subsection{Adaptive Tracking Controller for Ackermann Steering}\label{sec:ackermann_controller}
We apply \namepaper~to compensate for the sideslip when the robot is performing fast turning maneuvers. Thus, our adaptive control algorithm is applied only to the lateral and angular controller, although it can be applied for the linear velocity, as well.
We define the path error $\mathbf{e}=[e^\parallel, e^\perp]$ with the longitudinal and lateral error components, as seen in Fig.~\ref{fig:frames_of_reference}
\begin{equation}\label{eq:e}
\mathbf{e}  := \vp^\mathcal{D} - \vp_d^\mathcal{D}  = \mathbf{R}_\mathcal{I}^\mathcal{D}(\vp^\mathcal{I} - \mathbf{O}_\mathcal{D}^\mathcal{I}) 
\end{equation}
where $\mathbf{O}_\mathcal{D}^\mathcal{I}$ is the origin of the desired frame $\mathcal{D}$ expressed in $\mathcal{I}$, $\vp=[p_x, p_y]$ is the position of the robot, $\vp_d=[p_{x,d}, p_{y, d}]$ is the desired position from the trajectory, and $\mR_\mathcal{I}^\mathcal{D}$ is the rotation from the inertial frame $\mathcal{I}$ to the desired frame $\mathcal{D}$.
Next, we compute the time derivative of the path error~\eqref{eq:e} as follows
\begin{equation}\label{eq:e_perp_e_parallel}
\begin{aligned}
\begin{bmatrix}
\dot{e}^\parallel \\
\dot{e}^\perp
\end{bmatrix} 
 =\mR_\mathcal{I}^\mathcal{D} \mR_\mathcal{B}^\mathcal{I}\vv^\mathcal{B} -\vv_d^\mathcal{D}  + \dot{\mR}_\mathcal{I}^\mathcal{D} (\mR_\mathcal{D}^\mathcal{I} \vp^\mathcal{D}), \\
\end{aligned}
\end{equation}
where $\mR_\mathcal{B}^\mathcal{I}$ is the rotation from the body frame $\mathcal{B}$ to the inertial frame $\mathcal{I}$, $\vv_d^\mathcal{D} = [v_{d,x}^\mathcal{D}, 0]$ is the derivative of the desired position taken in $\mathcal{I}$, and expressed in $\mathcal{D}$, and $\vv^\mathcal{B} =[v_x^\mathcal{B}, v_y^\mathcal{B}]$ is the velocity of the robot in the $\mathcal{B}$ frame.
From~\eqref{eq:e_perp_e_parallel}, the perpendicular error derivative becomes
\begin{equation}\label{eq:e_dot}
\begin{aligned}
\dot{e}^\perp
= \begin{bmatrix}
\sin(\psi_e) \\
\cos (\psi_e)
\end{bmatrix} \cdot \vv^\mathcal{B}
- \dot{\psi}_d e^\parallel,
\end{aligned}
\end{equation}
where $\psi_e = \psi - \psi_d$ is the angle error between the actual orientation and the desired orientation.
In addition, each component in~\eqref{eq:e} is
\begin{equation}\label{eq:e_dot2}
e^\perp
= \begin{bmatrix}
\sin(\psi_e) \\
\cos (\psi_e)
\end{bmatrix} \cdot \tilde{\vp}^\mathcal{B},
\quad  e^\parallel
= \begin{bmatrix}
\cos(\psi_e) \\
-\sin (\psi_e)
\end{bmatrix} \cdot \tilde{\vp}^\mathcal{B},
\end{equation}
where $\tilde{\vp}^\mathcal{B} = \vp^\mathcal{B} - \vp_d^\mathcal{B}$ is the position error expressed in $\mathcal{B}$. 
Because we model the dynamics decoupled and linearized, from~\eqref{eq:e_dot}, we obtain 
\begin{equation}
\label{eq:e_perp}
\dot{e}^\perp = v_{y}^\mathcal{B} + v_x^\mathcal{B}\psi_e. \\
\end{equation}
Further, we differentiate~\eqref{eq:e_perp} and substitute \eqref{eq:simplified_dynamics_traxxas}, as follows
\begin{equation}\label{eq:e_ddot}
\ddot{e}^\perp  = -\frac{2C_y}{m v_x^\mathcal{B}} v_y^\mathcal{B} - v_x^\mathcal{B}\omega \revisiontwo{+} \frac{C_y}{m} u_\delta  \revisiontwo{+} \left(\boldsymbol{\phi}_1 \vtheta\right) u_\delta + \dot{v}^\mathcal{B}_x \psi_e + v_x^\mathcal{B}\omega_e.
\end{equation}
Now we design a tracking controller for the lateral motion of the vehicle.
Let $s^\perp = \dot{e}^\perp + k_p e^\perp$,  with $k_p \in \mathbb{R}^+$ a positive constant.
Then, using~\eqref{eq:e_ddot}, $\dot{s}^\perp$ is
$$\label{eq:s_dot}
\dot{s}^\perp = -\frac{2C_y}{m v_x^\mathcal{B}} v_y^\mathcal{B} - v_x^\mathcal{B}\omega \revisiontwo{+} \frac{C_y}{m} u_\delta  \revisiontwo{+} \left(\boldsymbol{\phi}_1 \vtheta\right) u_\delta + \dot{v}^\mathcal{B}_x \psi_e + v_x^\mathcal{B}\omega_e + k_p \dot{e}^\perp.
$$
We then design the following adaptive controller
\begin{equation}\label{eq:controller_traxxas_adapt}
u_\delta = \revisiontwo{-}\hat{b}_n^{-1}\left( k_v s^\perp - \frac{2C_y}{m v_x^\mathcal{B}} v_y^\mathcal{B} +  \dot{v}^\mathcal{B}_x \psi_e - v_x^\mathcal{B}\omega_d +  k_p \dot{e}^\perp
\right),
\end{equation}
where $\hat{b}_n= \frac{C_y}{m} + \boldsymbol{\phi}_1 \hat{\vtheta}$.
Letting $\tilde{\vtheta} = \hat{\vtheta} - \vtheta$, the closed-loop system of ${s}^\perp$ becomes
\begin{equation}\label{eq:s_traxxas}
\dot{s}^\perp  + k_v s^\perp= \left(\boldsymbol{\phi}_1 \tilde{\vtheta}\right) u_\delta.
\end{equation}
Note that the controller in~\eqref{eq:controller_traxxas_adapt} resembles~\eqref{eq:main_controller}, which was derived for the tracked vehicle. 
Using the same proof as in~\Cref{sec:control_synthesis}, we show that the tracking error $s^\perp$ and $\tilde{\vtheta}$ exponentially converge to a bounded error ball.
Next, we analyze the stability of the internal states $v_y^\mathcal{B}(t)$ and $\psi_e(t)$ under the exact error definitions~\eqref{eq:e_dot}-\eqref{eq:e_dot2}.}
\begin{theorem}\label{collorary2}
\revision{
    If $|s^\perp(t)| \le e^{-\gamma t}|s_0^\perp| + \frac{\epsilon}{\gamma}$, for positive constants $\gamma$ and $\epsilon$, \revisiontwo{and $s_0$ being the initial value}, 
    under the local assumption of $\revisiontwo{-}\frac{\pi}{2} < \psi_e < \frac{\pi}{2}$ and a positive $v_x^\mathcal{B}$, then $\tilde{p}^\mathcal{B}_y$, $v^\mathcal{B}_y$, and $\psi_e$ exponentially tend to bounds.
    }
\end{theorem}
\begin{proof} 
\revision{
    Our forward velocity controller ensures $v_x^\mathcal{B}$ converges to the desired forward velocity, as shown in~\Cref{theorem1} for the tracked vehicle.
    Thus, we can approximate $\tilde{p}^\mathcal{B}_x \approx 0$ in~\eqref{eq:e_dot2}. Hence,~\eqref{eq:e_dot} and~\eqref{eq:e_dot2} are simplified as
    \begin{equation} \label{eq:e_perp_simplified} 
        e^\perp  \approx \cos\psi_e \tilde{p}^\mathcal{B}_y,  \quad 
        \dot{e}^\perp \approx \sin\psi_e (v^\mathcal{B}_x + \dot{\psi}_d \tilde{p}^\mathcal{B}_y) + \cos\psi_e v^\mathcal{B}_y.
    \end{equation}
    Note that under no disturbance ($\epsilon \approx 0$) \revisiontwo{and} since $s^\perp = \dot{e}^\perp + k_pe^\perp$, $e^\perp$ and $\dot{e}^\perp$ exponentially converge to $0$. Assuming a feasible reference trajectory (nonzero desired side velocity, $v^\mathcal{B}_{y,d}$), since $\cos\psi_e$ and $v_x^\mathcal{B}$ are nonzero values, we have $\tilde{p}^\mathcal{B}_y$, $\tilde{v}^\mathcal{B}_y$, and $\psi_e$ converge to 0. 
    If $\epsilon$ is a small nonzero value, assuming $\inf_t(\cos\psi_e) = \bar{\psi}$, we can show that $|\tilde{p}^\mathcal{B}_y|$ exponentially converge\revisiontwo{s} to a small error bound, as follows
    \begin{align}
        \lim_{t \to \infty} |\tilde{p}^\mathcal{B}_y| &\le \frac{\epsilon}{\bar{\psi} k_p\gamma}.
    \end{align}
    We further assume that $|\tilde{p}^\mathcal{B}_y(t)| \approx e^{-\gamma_p t}|\tilde{p}^\mathcal{B}_y(0)| + \frac{|d(t)|}{k_p\gamma}$, where \revisiontwo{$\gamma_p$ is a positive constant and} $d(t)$ is a function with a small Lipschitz constant $\epsilon_y$. 
    With this assumption, we can show $v^\mathcal{B}_y$ exponentially converge\revisiontwo{s}  to the bound $\frac{\epsilon_y}{k_p\gamma}$.
    Then, assuming $\inf_tv^\mathcal{B}_x = \bar{v}$, with positive $\bar{v}$, we apply the triangle inequality for $\dot{e}^\perp$ as follows
    \begin{align}
        |\sin\psi_e (v^\mathcal{B}_x + \dot{\psi}_d \tilde{p}^\mathcal{B}_y)| &\le |\dot{e}^\perp| + |\cos\psi_e v^\mathcal{B}_y|.
    \end{align}
    Taking the limit and denoting $\bar{v} - \frac{\epsilon\sup_t|\dot{\psi}_d|}{\bar{\psi}k_p\gamma}$ as $\underbar{$v$}$, we show that $\psi_e$ exponentially converges to a bounded error as
    \begin{align}
        \lim_{t \to \infty} |\psi_e| \le \arcsin{\left(\frac{2\epsilon}{\underbar{$v$}\gamma} + \frac{\bar{\psi}\epsilon_y}{\underbar{$v$}k_p\gamma}\right)}. 
    \end{align}
    Note that $\gamma$ and $k_p$ can be chosen to make the error bounds sufficiently small. \revisiontwo{The proof above is based on~\eqref{eq:e_dot}. Using the simplified version of~\eqref{eq:e_perp}, the bound simplifies to $\lim_{t \to \infty} |\psi_e| \le \frac{2\epsilon}{\underbar{$v$}\gamma} + \frac{\epsilon_y}{\underbar{$v$}k_p\gamma}.$}
    }
\end{proof}

\section{Empirical Results: Analysis of VFM Suitability}
\label{sec:interprebility}

For our empirical work, we selected DINO V1~\cite{caron2021emerging} as the VFM.
DINO maps a high-resolution red-green-blue (RGB) image to a lower-resolution image where each pixel is a high-dimensional feature vector that depends on the \emph{entire} input image, not just the corresponding input patch.
More precisely,
let $\xi \in \N$ be the patch dimension and $\xi_f \in \N$ the feature vector dimension.
Given an RGB image $\mathcal{I}_{\mathrm{RGB}}$, the transformation is
\begin{equation}\label{eq:rgb_vfm}
\mathcal{I}_{\mathrm{RGB}}: h \times w \times 3 \rightarrow \mathcal{I}_{\mathrm{VFM}}: \floor*{\frac h \xi} \times \floor*{\frac w \xi} \times \xi_f, 
\end{equation}
where $h$ and $w$ are the image height and width.
We then extract prominent patches from $\mathcal{I}_{\mathrm{VFM}}$ to form the terrain representation $\cE$, which will be further used in the $\mphi$ from~\eqref{Eq:Phidef}.
\revision{For our experiments, we select a set of patches on right and left of the tracks/wheels of the vehicle, as emphasized in Fig.~\ref{fig:concept_overview}.}

DINO is optimized for a self-supervised learning objective and was shown to yield feature mappings useful for a variety of downstream tasks.
This VFM is trained on the \revision{ImageNet} dataset, which \revision{also} includes diverse ground terrains but mainly in the context of buildings, plants \revisiontwo{or} landscapes, instead of terrain-only images.
Therefore, in this section, we verify that DINO is able to clearly discriminate between different terrain types in terrain-only images before deploying it in our control setting.
\begin{figure}[t!]
    \centering
\includegraphics[width=\columnwidth]{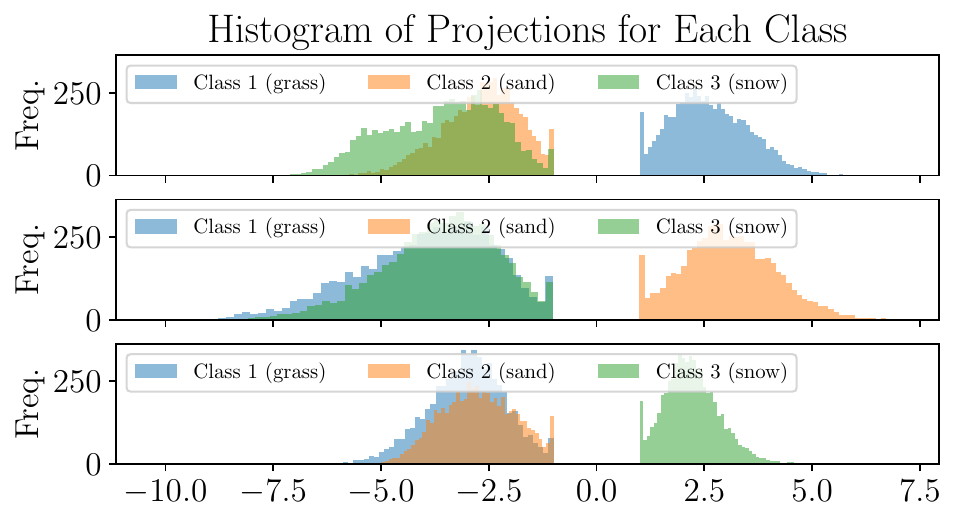}
    \caption{DINO VFM discriminative ability for different terrains. We show the histograms of the projection values onto the separating hyperplane normal computed using Support Vector Classifier for 3 sets of classes with 5 images each (each row presents the separation margin between one class type and the other 2 classes). \revisiontwo{Note that the spikes at -1 and 1 are an artifact of the high dimensionality and the small dataset we used.}
    }
    \label{fig:projections_svc}
\end{figure}

We first measure DINO's discriminative ability by examining the margins of linear classifiers between terrain classes in the high-dimensional feature space $\mathbb{R}^{\xi_f}$.
We consider three terrain types: grass, sand, and snow.
We collect five example images for each class and convert each image to a set of feature vectors using DINO.
Then, using the known class labels,
we fit a multi-class linear classifier for the feature vectors using the One-vs-Rest Support Vector Classifier (OVR-SVC) method~\cite{708428}.
We then project the feature vectors onto the one-vs-rest separating hyperplane normals.
Let the separating hyperplane have the equation $\vw_h\cdot \vx + b_h = 0$, where $\vw_h\in \mathbb{R}^{\xi_f}$ is the vector normal to hyperplane and $\mathbf{b}_h \in \mathbb{R}$ is the bias term.
Let $\cE$ be represented by just one patch, and thus have size $\xi_f$.
The projected patch $\cE$ onto the separating hyperplane normal is defined as $p_h = \vw_h \cdot \cE$. 
The histogram of these projected values for each patch in the image is shown in Fig.~\ref{fig:projections_svc}.
By comparing the SVC margin (the separation between -1 and 1) to the width of the histograms, we confirm that the classes are highly separable.

We next examine the distribution of the features across a sequence of images, taken while navigating from flagstone\revisiontwo{s} (irregular-shaped flat rocks) to gravel in the Mars Yard~\cite{mars_yard} at NASA Jet Propulsion Laboratory (JPL).
The top row of Fig.~\ref{fig:sequencial_images_dino} displays 8 out of a total of 45 images extracted from a video.
Each image is processed through the DINO \ac{vfm}, yielding $1200$ patches of dimension $\xi_f=384$ per image (computed using~\eqref{eq:rgb_vfm}).
We apply OVR-SVC on the patches from one flagstone and one gravel image and project all patches from our chosen 8 images onto the SVC separating hyperplane normal.
This projection reveals a bimodal distribution in the $3^{\mathrm{rd}}$ and $4^{\mathrm{th}}$ images due to the presence of both flagstones and gravel.
In the bottom subplot of Fig.~\ref{fig:sequencial_images_dino}, we simulate a scenario where the robot traverses the area covered in all 45 images sequentially.
For each image, we focus on a central patch of size $\xi_f$ and project these features onto the separating hyperplane normal.
This projection shows a consistent and continuous trend as the robot transitions from flagstone to gravel surfaces.
This observation ensures the continuity of the VFM with respect to the camera motion.

Overall, these results provide positive empirical evidence that the DINO VFM is suitable for fine-grained discrimination of terrain types in images containing only terrain, and thus suitable for use in our setting. 

\begin{figure}[t]
    \centering\includegraphics[width=0.8\columnwidth]{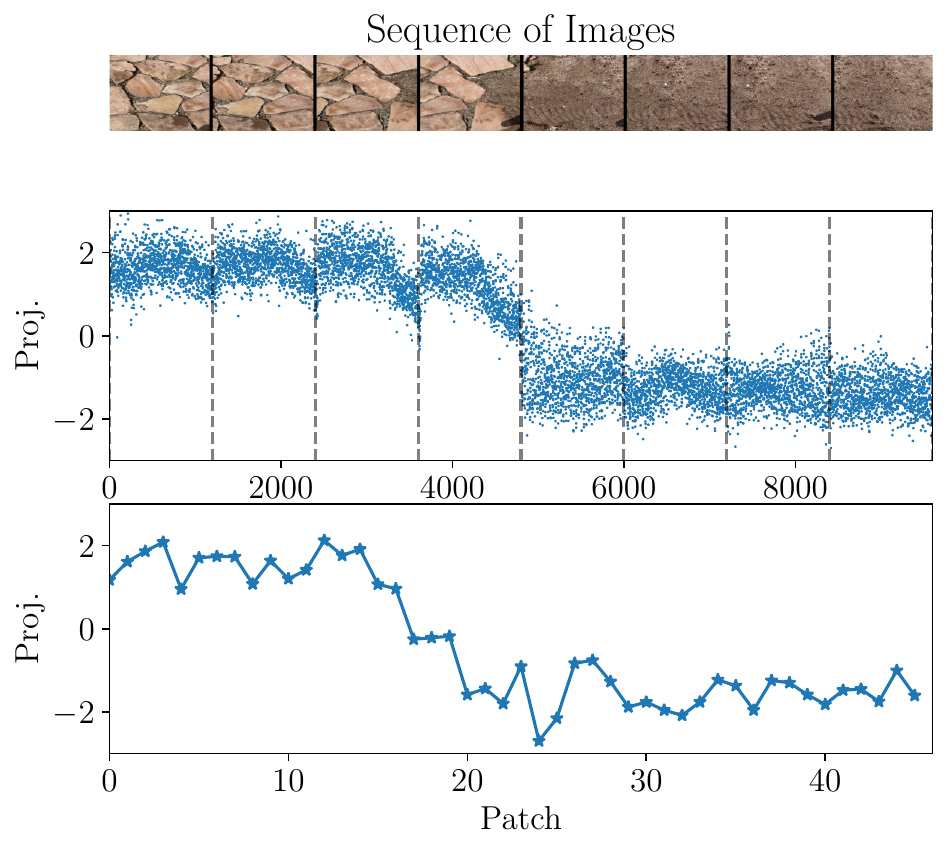}
    \caption{
        Projection of sequential flagstone\revisiontwo{s} and gravel features onto an OVR-SVC separating hyperplane normal. The middle plot shows the projection of all the patch features from the top 8 figures, while the bottom plot shows the projection of a central patch taken from 45 sequential images of flagstone\revisiontwo{s} and gravel.}
    \label{fig:sequencial_images_dino}
\end{figure}

\section{Empirical Results: Simulation Studies}\label{sec:simulation}

\subsection{Simulation Study Settings}

To validate our learning and control strategy, we developed a simulation environment (Fig.~\ref{fig:toy_environment})
that enables detailed visualizations of the algorithm behavior.
The dynamics for the simulator were modeled via \eqref{eq:dynamicsAB}, and the controller of~\eqref{eq:main_controller} \revisiontwo{and}~\eqref{eq:adapt_law} with the coefficients in Table~\ref{tab:control_coeffs_sim} was used to track user-defined velocity trajectories $\mathbf{v}_\mathrm{ref}$ generated at random.
The environment contains three distinct terrain types (Fig.~\ref{fig:toy_environment}a).
Each terrain type induces a different level of slip, modeled as a scaling of the \revision{nominal control} matrix $\mathbf{B}_\mathrm{n}\revision{\in \mathbb{R}^{2\times2}}$ in \eqref{eq:dynamicsAB} such that $\mB_\mathrm{n}$ is replaced by $\eta \mB_\mathrm{n}$ \revisiontwo{and the dynamics matrix} \revision{$\mA\in\mathbb{R}^{2 \times 2}$ is kept the same as in \eqref{eq:dynamicsAB}.}

\begin{figure}[t]
    \centering
    \begin{minipage}[b]{0.49\linewidth}
        \includegraphics[width=0.9\linewidth]{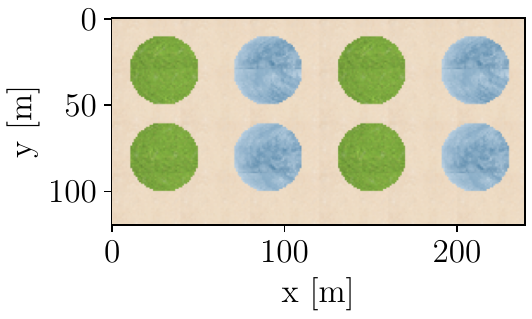}
        \caption*{\quad \quad (a)}
    \end{minipage}
    \hfill %
    \begin{minipage}[b]{0.49\linewidth}
        \includegraphics[width=0.9\linewidth]{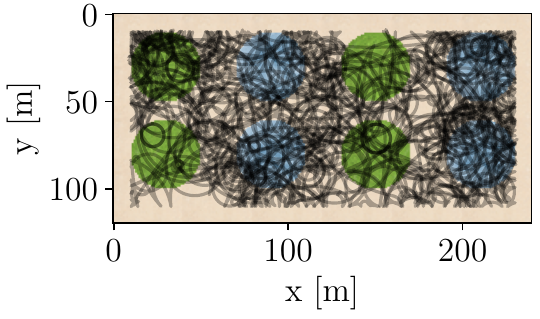}
        \caption*{\quad \quad  (b)}
    \end{minipage}
    \caption{(a) Environment with 3 different types of terrain (sand, grass, and ice), which represent areas of differing slip \revisiontwo{coefficients} (b) Generated trajectories for training.}
    \label{fig:toy_environment}
\end{figure}

\begin{table}[t]
\caption{Control and adaptation coefficients for the $\mphi$ constant and $\mphi$ \ac{dnn} controllers for the simulation environment\revisiontwo{.}}
    \label{tab:control_coeffs_sim}
    \centering
    \begin{tabular}{c|c|c|c|c|c|c}
        \hline
          Ctrl. Type & $k_{dx}$ & $k_{d \omega}$ &  $\mgamma_0$ diag   & $\mathbf{R}$ diag  & \revision{$\mathbf{Q}$ diag} & \revision{$\lambda$} \\ \hline \hline
          $\mphi$ = ct. & $0.05$ & $0.1$  & $0.01$ & $0.1$ & \revision{$1.0$}   & \revision{$0.01$} \\   \hline
          $\mphi$ = NN & $0.05$ & $0.1$  & $0.01$ & $0.1$ & \revision{$1.0$} & \revision{$0.01$} \\ \hline
    \end{tabular}
\end{table}
To construct a dataset, we simulate $N = 1$ long trajectory of 150\,000 discrete time steps, with randomized piecewise-constant velocity inputs.
For acquiring these features, we utilize the DINO \ac{vfm} on images of the terrain\revisiontwo{s from~\Cref{fig:toy_environment}a}.
As explained in~\Cref{sec:interprebility}, this model processes high-resolution terrain images into a more compact, lower resolution embedding.
This reduced resolution representation is \revision{overlaid} across the entire map.
Specifically, let $m_s \times n_s$ be the size of the simulated map, which is $120 \times 240$ in our case. Let each DINO feature image have the size computed as in~\eqref{eq:rgb_vfm}, \revision{for a background image of size $480 \times 640$, a patch size of $\xi = 16$, and \revisiontwo{the feature size} $\xi_f = 384$.
\begin{equation}
\mathcal{I}_{\mathrm{VFM}}: \floor*{\frac {480}{16}} \times \floor*{\frac{640}{16}} \times 384 = 30 \times 40 \times 384.
\end{equation}
}
\revision{Then}, we tile each of the DINO feature images across the entire map \revision{vertically 4 times $\left(\floor*{\frac {120}{30}}\right)$ and horizontally 6 times $\left(\floor*{\frac {240}{40}}\right)$}
and extract and record the relevant terrain features underneath the robot.
For training, we collect random trajectories, generated by sampling control inputs $\mathbf{u}$ from a uniform distribution, and integrate forward the dynamics in~\eqref{eq:dynamicsAB} in order to cover a large portion of the simulated map, as seen in Fig.~\ref{fig:toy_environment}b.
The dataset contains the DINO features extracted from underneath the robot and the robot's velocities, and as labels the residual dynamics derivative $\vy$\revision{, computed as in~\eqref{eq:dyn_res_tracked_vehicle}}.
Using this dataset, we then train the basis function $\mphi$, whose architecture can be seen in Fig.~\ref{fig:dnn_architecture}, using Algorithm~\ref{alg:training-v2}.
This compact representation of the terrain is then integrated with online adaptive control (Algorithm~\ref{alg:adaptation_at_execution_time}).

\subsection{Simulation Study Objectives}

In exploring the capabilities of our model, we investigate how prior knowledge of the terrain contribute\revision{s} to improved tracking accuracy for an adaptive controller.
We thus test our algorithm across 3 scenarios:
a) We assess the model performance in an environment identical to the one used during training to understand its effectiveness with in-distribution data (Fig.~\ref{fig:toy_model_1}).
b) We test the algorithm under simulated nighttime conditions to gauge performance when the ground is identical, but the lighting conditions are different (Fig.~\ref{fig:toy_model_night}).
c) We challenge the model by presenting it with two environments that have similar visual features to those in the training data set, but exhibit different dynamic behaviors. Furthermore, we adopt an adversarial approach by exposing the robot to completely novel environments that are not encountered during training (Fig.~\ref{fig:toy_model_adversarial}).

\begin{figure}[t]
    \centering
    \includegraphics[width=0.9\columnwidth]{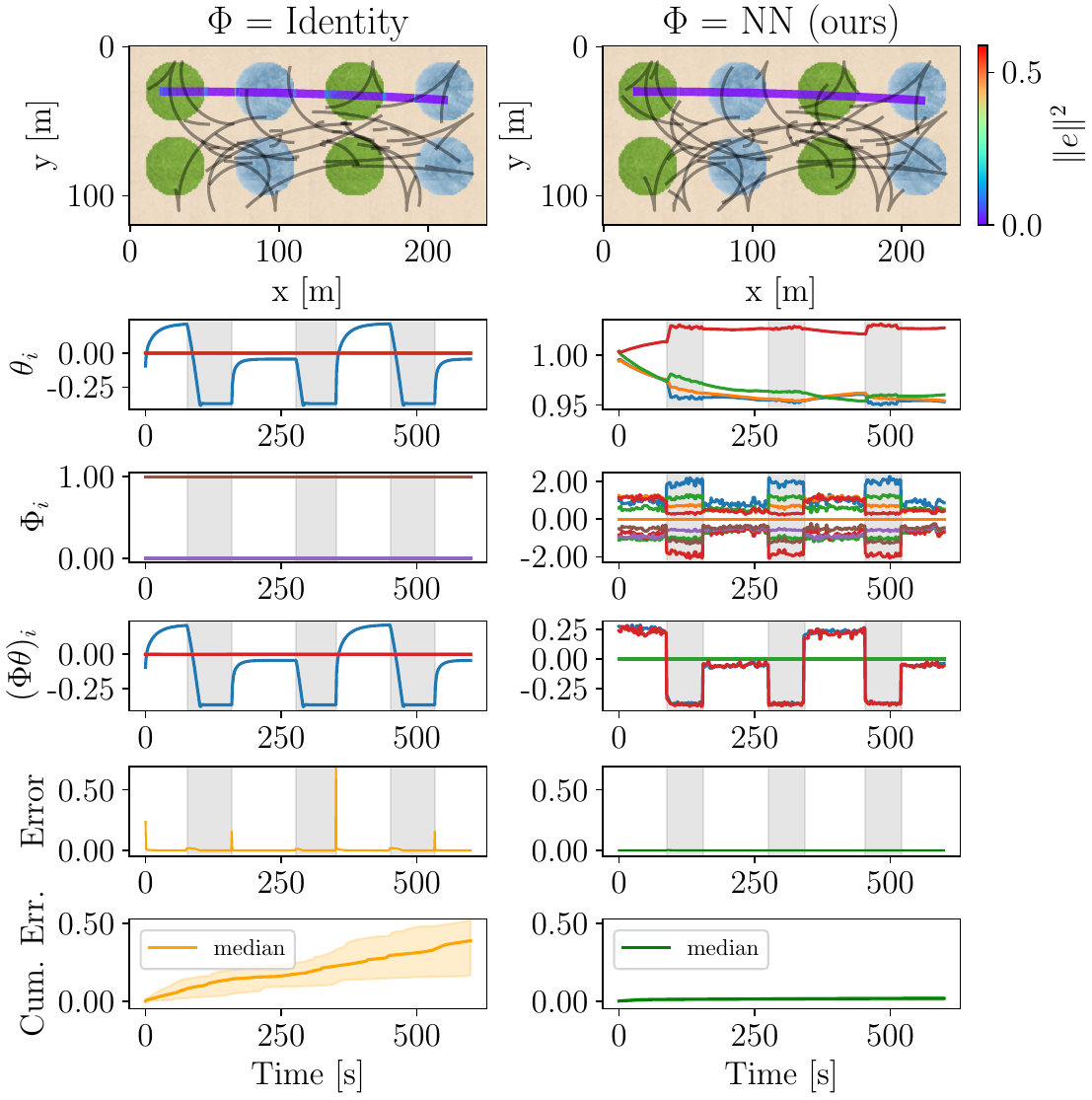}
    \caption{Results for in-distribution data \revisiontwo{from} the simulation model. Each experiment was run 40 times \revisiontwo{on the terrains from row 1}. The left column contains the performance for the baseline controller, while the right column contains the performance for our method. 
    \revision{The second row contains the adaptation coefficients $\hat{\theta}_i$, while the third row emphasizes the basis function $\mphi_i$. For the baseline, $\mphi_i$ is constant, while in our method, $\mphi_i$ varies as a function of the terrain and state.}
    The last row presents the cumulative error, where the thick colored line represents the median, and the shaded region encompasses the range from the 25\textsuperscript{th} to the 75\textsuperscript{th} percentile.}
    \label{fig:toy_model_1}
\end{figure}
\subsection{In-distribution Performance}
To quantify if prior knowledge of the terrain improves tracking accuracy, the robot is tested in-distribution using the same environment as in the training dataset.
The first row of Fig.~\ref{fig:toy_model_1} shows 39 random trials (black) and the single exemplar path (\revisiontwo{shades of} purple, colored by the $\mathcal{L}_2$ error between the actual and desired states). These random trials are used to compute error statistics in row 6.
The second row displays the robot's adaptation coefficients for the purple trajectory as it navigates through this environment.
When the basis function lacks terrain awareness and is \revisiontwo{set as constant matrices}~\eqref{eq:constant_phi}, 
\begin{equation}\label{eq:constant_phi}
    \mphi = \left\{
        \scriptsize
        \begin{bmatrix}
            1 & 0 \\
            0 & 0
        \end{bmatrix},
        \begin{bmatrix}
            0 & 1 \\
            0 & 0
        \end{bmatrix},
        \begin{bmatrix}
            0 & 0 \\
            1 & 0
        \end{bmatrix},
        \begin{bmatrix}
            0 & 0 \\
            0 & 1
        \end{bmatrix}
    \right\}.
\end{equation}
there is significant fluctuation in the adaptation coefficients during the transition between different terrains. Conversely, when the \ac{dnn} basis function is used, the adaptation coefficients remain relatively stable, while the \ac{dnn} output itself varies with each terrain type, as \revision{shown in the third row of Fig.~\ref{fig:toy_model_1}}. 
The fourth row showcases the \revision{components of} the product between the \ac{dnn} basis function \revision{$\mphi$} and the adaptation vector \revision{$\hat{\vtheta}$}.
Though the output of our controller is slightly noisier, the adaptation of the \revision{product $\mphi \hat{\vtheta}$} is significantly faster.
The fifth row shows the normalized error between the robot's actual and desired states. For the constant basis function, most of the tracking error occurs at terrain transitions. When the basis function is terrain-informed, the error is \revisiontwo{negligible}, even at terrain transitions.
Finally, the last row shows the spread of the cumulative error across 40 distinct experimental runs, each initiated at a random starting point and orientation, but of \revisiontwo{the} same duration (the black and the purple trajectories in Row 1). The results of the simulation show that our terrain-informed \ac{dnn}-based tracking controller reduces the cumulative error by approximately 90.1\% when compared to the constant $\mphi$, define\revisiontwo{d} as in~\eqref{eq:constant_phi}.

\subsection{Nighttime Out-of-distribution Performance}

To test the robustness of our framework to varying lighting conditions, we extend our simulated experiments with a nighttime environment by uniformly darkening (changing the brightness) of each image representing the environment (Fig.~\ref{fig:toy_model_night}).
While the adaptation coefficients exhibit more variation compared to those in the standard, in-distribution scenario, the \ac{dnn} still demonstrates good accuracy in predicting the environment from the darkened images. This outcome emphasizes the robustness of \ac{vfm}, underscoring \revisiontwo{its} ability to adapt effectively to varying lighting conditions. Importantly, even in these altered night conditions, the cumulative error remains low. 

\begin{figure}[t]
    \centering\includegraphics[width=0.9\columnwidth]{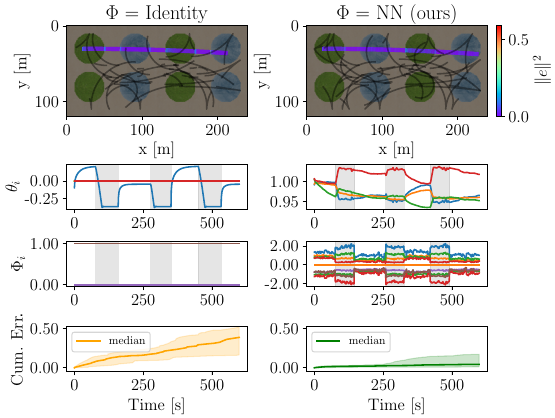}
    \caption{Simulation results in the simulated nighttime environment. Despite different lighting conditions, the cumulative error is kept small by the terrain-informed \ac{dnn}.
    }
    \label{fig:toy_model_night}
\end{figure}

We highlight the importance of adaptation by comparing the tracking error under two scenarios: with adaptation and without adaptation, \revisiontwo{as shown in} Fig.~\ref{fig:toy_model_sim_adapt_coeffs}. In the ``no adaptation case,'' we maintain $\boldsymbol{\theta}$ as a constant, initialized to $\vtheta_r$.
This comparison effectively demonstrates the benefits of adaptation, emphasizing its value even in situations where the basis function accurately predicts the terrain.

\begin{figure}[t]
    \centering\includegraphics[width=0.8\columnwidth]{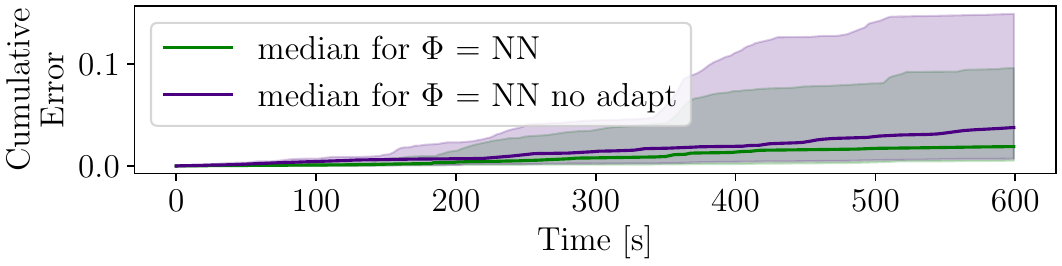}
    \caption{Comparison of the cumulative error between adaptation and ``no adaptation'' for the simulated nighttime experiment. In both cases, the \ac{dnn} version of $\mphi$ is used. \revisiontwo{This figure emphasizes the benefit of doing online adaptation.} }\label{fig:toy_model_sim_adapt_coeffs}
\end{figure}

\subsection{Adversarial Environment Performance}

In our final test (Fig.~\ref{fig:toy_model_adversarial}), we introduced two adversarial environments for the robot, manipulating two visually similar environments by altering their respective $\eta$ coefficient of the $\mathbf{B}$ matrix. 
This emulates the real world where pits of deep sand appear very similar to shallow sand, but have a significantly different effect on the dynamics of the robot.
Additionally, we modified the appearance of the simulated ice environment to create a distinct visual difference, while also slightly changing the effect of ice on the dynamics.

In the adversarial environment, the adaptation coefficients exhibit greater changes than for the in-distribution and night-time simulations.
In addition, we observe that the \ac{dnn} basis function demonstrates good performance, validating its effectiveness in handling out-of-distribution data. This effectiveness is likely attributed to the zero-shot capability inherent in the \ac{vfm}.
Lastly, it is important to note that the overall cumulative error remained lower compared to scenarios where the basis function lacked terrain information, further demonstrating the benefit and robustness of our approach in varied and challenging conditions, even for out-of-distribution data.

\section{Empirical Results: Hardware Experiments}\label{sec:experimental}
We focus on the hardware implementation and experimental validation of our~\namepaper~adaptive controller discussed in~\Cref{sec:terrain_informed_adaptive}-\ref{sec:interprebility} on a tracked vehicle whose dynamics are modeled in~\eqref{eq:dynamicsAB} \revision{and a car with Ackermann steering with the dynamics modeled as in~\eqref{eq:simplified_dynamics_traxxas}}.
We \revision{present} how our adaptive controller effectively addresses various perturbations such as terrain changes, severe track degradation, and unknown internal robot dynamics.
\begin{figure}[t]
    \centering\includegraphics[width=0.9\columnwidth]{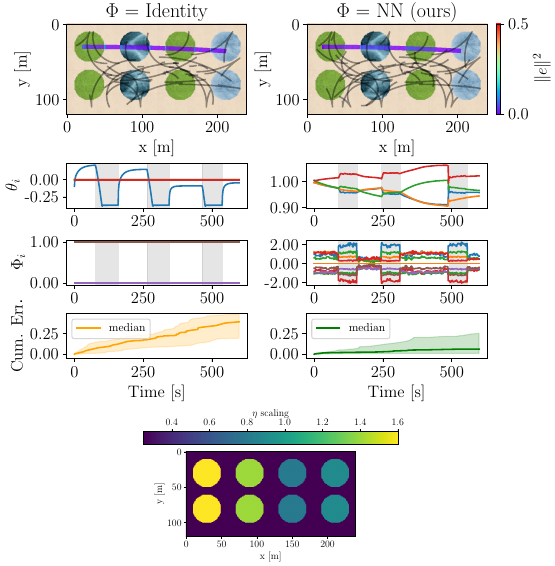}
    \caption{Simulation results with adversarial \revisiontwo{environment} (different \revisiontwo{structures that look like ice}) and different $\eta$ \revision{scaling} coefficient for the \revision{control matrix} for \revisiontwo{identical} terrain.}
    \label{fig:toy_model_adversarial}
\end{figure}

\subsection{Robot Hardware and Software Stack}
Experiments were carried out using a {GVR-Bot}~\cite{us2019gvr} \revision{ and a modified Traxxas X-Maxx, both shown in~\Cref{fig:mars_yard}.  Both vehicles are} equipped with an NVIDIA Jetson Orin, RealSense D457 cameras (GVR-Bot: two forward facing and one rear facing, Traxxas: single forward facing) and a VectorNav VN100 \revisiontwo{\ac{imu}}.

State estimation is provided onboard using OpenVINS~\cite{Geneva2020ICRA}, which fuses the camera data with an \revisiontwo{\ac{imu}} to estimate the platform's position, attitude, and velocity.
Our~\namepaper~controller, as presented in~\eqref{eq:main_controller},~\eqref{eq:adapt_law} \revisiontwo{for the tracked vehicle and in~\eqref{eq:controller_traxxas_adapt} for the car-like vehicle}, and Theorem~\ref{theorem1}, \revisiontwo{runs} at 20~Hz. It is implemented in Python using the Robot Operating System (ROS) as the middleware to communicate with the \revision{robot's} internal computer.

\subsection{Experiments on Slopes in JPL's Mars Yard}\label{sec:experiments_mars_slopes_JPL}
\revision{The {GVR-Bot} only accepts velocity commands as the track velocities are regulated using an internal PID controller, which is inaccessible to the user. While this justifies our first-order modelling \eqref{eq:dynamicsAB} using velocities, these experimental results validate that \namepaper~successfully learns the unknown internal dynamics.}
To verify the performance of our~\namepaper~controller (\Cref{sec:terrain_informed_adaptive} and~\ref{sec:modeling}) \revision{on different terrains}, the {GVR-Bot} was driven on the slopes of the Mars Yard~\cite{mars_yard} at the Jet Propulsion Laboratory (JPL). Fig.~\ref{fig:mars_yard} shows the two selected slopes, both chosen for their appropriate angle and visually different terrain type that induce different dynamic terrain-based behaviors.

\begin{figure*}[t]
    \centering
    \includegraphics[width=0.9\textwidth]{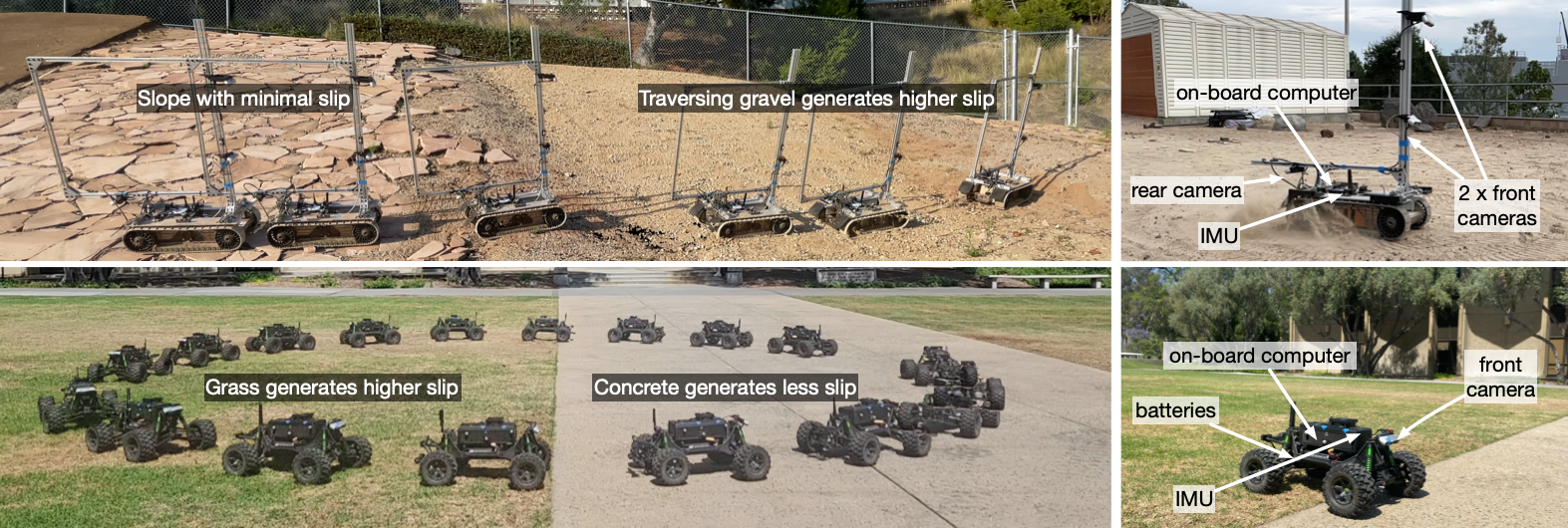}
    \caption{a) The {GVR-Bot} traversing two slopes with different textures and terrain-induced dynamic behaviors at the JPL Mars Yard. \revision{b) The {GVR-Bot} with the sensing and compute units highlighted. Note that the forward facing camera is used for state estimation, while the top camera is used for taking terrain images for~\namepaper. The rear camera is not used in this work.
    c) The Traxxas robot traversing two different terrains that induce different dynamic behaviors. d) The Traxxas robot with its main sensing and compute units highlighted.
    }
    }
\label{fig:mars_yard}
\end{figure*}
 
\subsubsection{Offline Training} \label{sec:offline_training_hardware}
Training data was collected by driving the {GVR-Bot} via direct tele-operation for a total of 20 minutes on the slopes.
This trajectory was designed to include segments of transition between different slopes as well as periods of single slope operation.
We utilize this dataset for training our terrain-dependent basis function as outlined in Algorithm~\ref{alg:training-v2}.
By leveraging the strengths of a pre-trained \ac{vfm}, we develop the lightweight \ac{dnn} basis function head used in the adaptive controller of \eqref{eq:main_controller},~\eqref{eq:adapt_law}.
This function processes inputs comprising of the mean of two visual feature patches from the {GVR-Bot}'s right and left tracks \revisiontwo{and}  the robot's velocity taken from the onboard state estimator.
The \ac{vfm}-based \ac{dnn} ($\boldsymbol{\Phi}$) structure incorporates two hidden layers, each consisting of 200 neurons, as seen in Fig.~\ref{fig:dnn_architecture}.
The output has size 16, which is then reconfigured into dimensions $n \times m \times  n_\theta$, where $n=2$ is the state size, $m=2$ is the control input \revision{size}, \revisiontwo{and} $n_\theta = 4$ is the size of the adaptation vector \revision{\revisiontwo{that matches} the number of terms in the control matrix}.
The hyperparameters for the training algorithm are shown in Table~\ref{tab:hyperparameters}.

\begin{table}[t]
    \caption{ Training hyperparameters for Algorithm~\ref{alg:training-v2} \revisiontwo{for the tracked vehicle}. $\beta$ is the learning rate for \revisiontwo{the optimization in} Line~\ref{alg_step:optimization} of Algorithm~\ref{alg:training-v2}, $\boldsymbol{\theta}_r$ is the regularization target,  $\ell_{\mathrm{min}}$ and $\ell_{\mathrm{max}}$ are the bounds for the distribution over trajectory window lengths, $\lambda_\mathrm{r}$ is the regularization term for~\eqref{eq:theta-for-window}, $K$ is the minibatch size, and $n_\theta$ is the size of the adaptation vector.}
        \centering
        \begin{tabular}{c|c|c|c|c|c|c}
        \hline
        $\beta$ & $\boldsymbol{\theta}_r$ & $\ell_{\mathrm{min}}$ & $\ell_{\mathrm{max}}$ & $\lambda_\mathrm{r}$ & $K$  & $n_\theta$ \\ \hline \hline
        0.001                  & $\mathbf{1}_4$                                                & 1.2 [s]         & 30 [s]          & 0.1                & 70    & 4     \\ \hline
        \end{tabular} \label{tab:hyperparameters}
\end{table}

\subsubsection{Online Adaptation} 

At runtime, the downward-facing camera\footnote{To mitigate the purple tint in the RGB images (a common issue for Intel Realsense cameras), the RGB cameras were outfitted with neutral density filters to maintain the integrity of the VFM features.} is used to capture images of the terrain at 20~Hz. These images are then processed by the \ac{vfm} explained in~\Cref{sec:interprebility} to extract the features. 
The extracted features are then concatenated with the robot's velocity and are then fed into the \ac{dnn} basis function $\mphi$. This function, together with an online-adapting vector, is then employed to dynamically adjust the residual $\mathbf{B}$ matrix~\eqref{eq:main_controller},~\eqref{eq:adapt_law} in real time to account for the different terrains.

The benefits of the terrain-informed basis function  can been seen by comparing the performance of a constant and non-constant basis function controller as the robot traverses slopes. Both controllers are based on \eqref{eq:main_controller} and the adaptation law in \eqref{eq:adapt_law}. The first controller uses a constant basis function, defined in~\eqref{eq:constant_phi}. \revision{We choose this structure for the constant $\mphi$ to capture both the direct and cross-term effects on the robot's velocity.} \revision{T}he second controller uses a terrain-dependent \ac{dnn} basis function trained as explained in~\Cref{sec:offline_training_hardware}.
The control coefficients for both controllers are presented in Table~\ref{tab:control_coeffs_slopes}.
The initial adaptation vector $\boldsymbol{\theta}_0= \boldsymbol{\theta}(0)$ for the constant basis function is $\mathbf{0}_{n_\theta}$, while $\boldsymbol{\theta}(0)$ for the terrain-dependent basis function is the converged value from Algorithm~\ref{alg:training-v2}.
\begin{table}[t]
    \caption{Control coefficients for both controllers ($\mphi$ is constant and $\mphi$ is a \ac{dnn}) \revision{for the tracked vehicle}.}
        \label{tab:control_coeffs_slopes}
        \centering
        \begin{tabular}{p{0.31cm}|p{0.31cm}|p{0.31cm}|p{0.31cm}|p{0.31cm}|p{0.95cm}|p{0.95cm}|p{0.95cm}|p{0.4cm}}
\hline 
$k_{px}$ &  $k_{py}$ & $k_{dx}$ & $k_{d \omega}$ &  $k_{\psi}$  &  $\mgamma_0$ diag & $\mathbf{Q}$ diag & $\mathbf{R}$ diag & $\lambda$ \\ \hline \hline
$0.8$ & $0.8$ & $0.5$ & $1.6$ &$2.3$ & $0.2$ & $0.1$ & $5.0$ & 0.01\\ \hline
\end{tabular}
\end{table}

Each experiment was carried out five times, with the results detailed in Fig.~\ref{fig:tracking_error_plot}. For repeatability, we used a rake to re-distribute the gravel on the slopes between runs and alternate\revisiontwo{d} back and forth between \revisiontwo{running} the two controllers.
For this experiment, the desired trajectory is a straight line that spans the entire length of the two slopes (see Fig.~\ref{fig:mars_yard}).

Fig.~\ref{fig:tracking_error_plot} shows that when the robot traverses the first slope (flagstone resulting in minimal slippage), both controllers have comparable tracking errors.
However, a notable change in performance appears when the robot transitions to the second slope, which has an increased tendency for the soil to slump down the hill, causing slippage.
In Table~\ref{tab:tracking_error_mars_yard}, we present the \ac{rmse} between the actual position and the desired position computed as $\sqrt{\frac{1}{L} \sum^{L}_{i = 1}||\mathbf{p}^{\mathcal{I}}_{i} - \mathbf{p}^{\mathcal{I}}_{di}||^2_2}$, where $L$ is the length of the trajectory.
The results demonstrate that the integration of a \ac{vfm} in an adaptive control framework enhances tracking performance, yielding an average improvement of 53\%.
\begin{figure}[t]
    \centering
    \includegraphics[width=0.9\columnwidth]{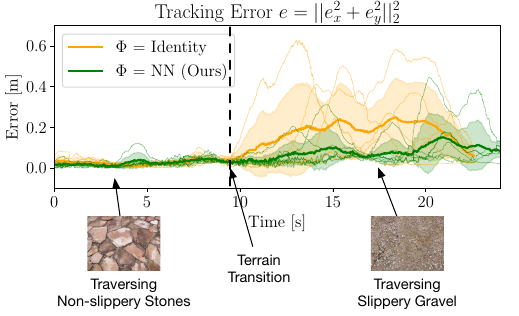}
    \caption{Tracking error for the two controllers (constant basis function and terrain-dependent basis function) on the slopes \revision{for a tracked vehicle}. The error is computed as the 
    Euclidean distance between actual and desired positions in the $\mathcal{I}$ frame. For both colors, the thick line outlines the mean of the 5 experiments, the shaded area \revisiontwo{represents} 1 standard deviation, and the thin and transparent lines denote the 5 experiments.}
    \label{fig:tracking_error_plot}
\end{figure}

\begin{table}[t]
    \caption{Statistics \revision{for the tracked vehicle} on Mars Yard slopes.}
        \centering
        \begin{tabular}{c|c}
        \hline
        Controller & Tracking error \revision{(}RMS \revision{[m]}\revision{)} \\ \hline \hline
        $\mphi $ constant & $0.130 \pm 0.038$ \\ \hline
         $\mphi $ \ac{dnn} (ours) & $0.061 \pm 0.022$ \\
        \hline
        \end{tabular} \label{tab:tracking_error_mars_yard}
\end{table}

\subsubsection{Computational Load}

A significant bottleneck in deploying \ac{vfm}s onboard robots is the computational requirements of the inference stage of the models, especially as typical controllers need to run at 10s-100s Hz.
To minimize the inference time and allow high controller rates, we employ the smallest visual transformer architecture of the DINO V1, consisting of 21 \revision{million network} parameters. This architecture allows us to run the controller at 20~Hz on the Graphics Processing Unit (GPU) \revision{on-board an NVIDIA Jetson Orin}.

\subsection{\revision{Experiments On-board an Ackermann Steering Vehicle}}\label{sec:experiments_traxxas}

\revision{We performed similar experiments to those described in~\Cref{sec:experiments_mars_slopes_JPL} using an Ackermann steering vehicle.
Here, the robot traverses two different terrains, as seen in~\Cref{fig:mars_yard}, which induce different dynamic behaviors onto the robot (grass is more slippery than concrete).
\revisiontwo{Our experiments on both vehicles showed that, on flat ground,} the car experiences more significant slippage and terrain disturbances compared to the tracked vehicle. \revisiontwo{Therefore, for this experiment, we validated~\namepaper~on flat ground.}
}

\revision{
In~\Cref{fig:plot_1_traxxas_baseline}, we show the product $\mphi \hat{\vtheta}$ for the constant basis function of the nonlinear tracking controller in~\eqref{eq:controller_traxxas_adapt}. 
As the robot transitions between the two terrains, we see that the robot effectively adapts to each terrain during this transition. 
This behavior mirrors that observed in the simulation plots (\Cref{fig:toy_model_1}).
Note that we maintained $n_\theta = 4$, to be consistent with the \ac{dnn} model of the basis function, even though all four parameters are identical in this instance.
}
\revision{In Fig.~\ref{fig:adaptation_traxxas_NN}, we emphasize the adaptation coefficients (left) and the \ac{dnn} basis function output (right) for the nonlinear tracking controller in~\eqref{eq:controller_traxxas_adapt} as the robot transitions between the two terrains (grass and concrete) several times.
The \ac{dnn} basis function switches depending on the type of environment it operates in, while the corresponding adaptation coefficients $\hat{\vtheta}$ \revision{remain} mostly constant.
This behavior also mirrors that observed in the simulation plots (\Cref{fig:toy_model_1}) when a \ac{dnn} with \ac{vfm} is employed.
Lastly, in~\Cref{fig:trajectories_traxxas}, we present the lateral position error ($e^\perp$) and lateral velocity ($v_y^\mathcal{B}$) for the 3 controllers (a) nonlinear PD (\eqref{eq:controller_traxxas_adapt} without the adaptation), (b) MAGIC with constant $\boldsymbol{\Phi}$ in \eqref{eq:controller_traxxas_adapt} \revisiontwo{and}\eqref{eq:adapt_law}, (c)~\namepaper~with \ac{dnn} $\boldsymbol{\Phi}$ in \eqref{eq:controller_traxxas_adapt} \revisiontwo{and}~\eqref{eq:adapt_law}. 
Our method shows superior performance compared to the baseline nonlinear PD controller. 
The control coefficients for the three controllers are outlined in~\Cref{tab:control_coeffs_ackermann}.
}

\begin{figure}[t]
    \centering
\includegraphics[width=0.95\columnwidth]{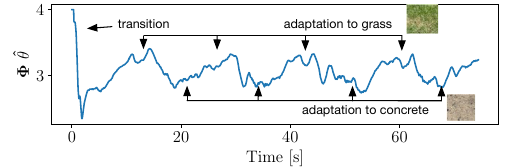}
    \caption{\revision{ $\mphi\hat{\vtheta}$ of~\eqref{eq:controller_traxxas_adapt} for the constant basis function baseline for the car-like robot. Different terrains induce convergence to different adaptation coefficients.}}
\label{fig:plot_1_traxxas_baseline}
\end{figure}

\begin{figure}[t]
    \centering
\includegraphics[width=0.95\columnwidth]{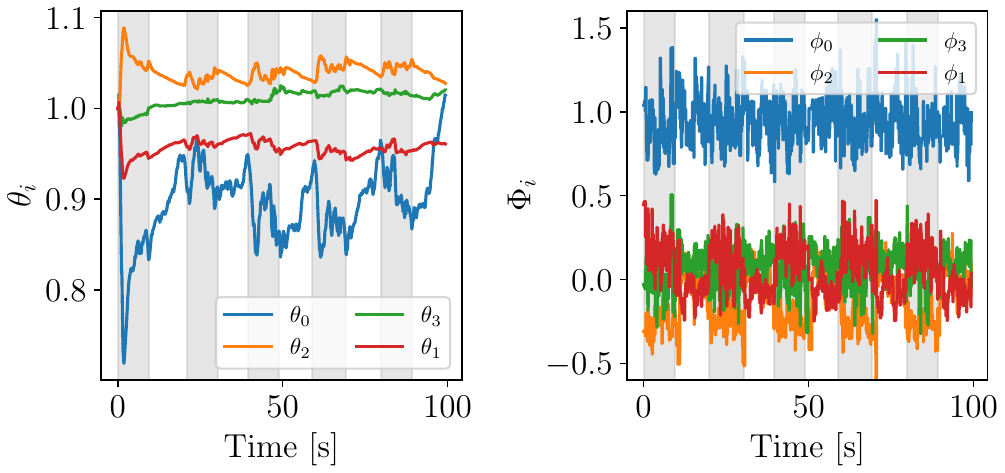}
    \caption{
        \revision{Adaptation coefficients $\hat{\vtheta}$ and terrain-dependent \ac{dnn} basis function for a car-like robot. Approximate under-vehicle terrain is denoted with the white (concrete) and gray (grass) bars.}
    }
\label{fig:adaptation_traxxas_NN}
\end{figure}

\begin{figure}[t]
    \centering
\includegraphics[width=0.95\columnwidth]{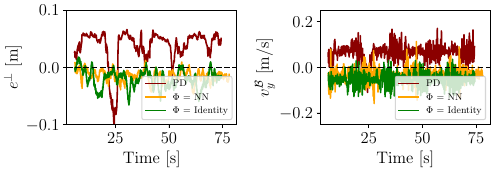}
    \caption{\revision{Convergence of the lateral error $e^\perp$ and lateral velocity $v_y^\mathcal{B}$ for a circular trajectory traversing two terrains like the one seen in~\Cref{fig:mars_yard} with $v_x^\mathcal{B} = 1.5~m/s$. The \revisiontwo{ velocity of the desired trajectory is} limit\revisiontwo{ed}~\revisiontwo{by} the performance of the \ac{vio} at higher speeds.}}
\label{fig:trajectories_traxxas}
\end{figure}

\begin{table}[t]
    \caption{\revision{Control coefficients for the lateral control \revisiontwo{of} the Ackermann steering vehicle}.}
        \label{tab:control_coeffs_ackermann}
        \centering
        \begin{tabular}{p{2.0cm}|p{0.3cm}|p{0.3cm}|p{0.9cm}|p{0.9cm}|p{0.9cm}|p{0.3cm}}
        \hline
            \revision{Controller} & \revision{$k_p$} &  \revision{$k_d$} &  \revision{$\boldsymbol{\Gamma}$ diag.}  & \revision{$\mathbf{Q}$ diag.} & \revision{$\mathbf{R}$ diag.} & \revision{$\lambda$}  \\ \hline \hline
            \revision{(a) nonlinear PD} & \revision{$1.0$} & \revision{$1.0$} & \revision{-} & \revision{-} & \revision{-} & \revision{-}  \\ \hline
            \revision{(b) $\mphi$ constant} & \revision{$1.0$} & \revision{$1.0$} & \revision{$1.5$} & \revision{$1.0$} & \revision{$0.01$} & \revision{$0.05$} \\ \hline
            \revision{(c) $\mphi$ \ac{dnn}} & \revision{$1.0$} & \revision{$1.0$} & \revision{$1.5$} & \revision{$1.0$} & \revision{$0.01$} & \revision{$0.05$} \\ \hline
        \end{tabular}
\end{table}

\subsection{Indoor Track Degradation Experiments}

Indoor experiments were conducted at Caltech's Center for Autonomous Systems and Technologies (CAST) (Fig.~\ref{fig:cast_experiment}).
The primary objective of these experiments was to evaluate the robustness and performance of our proposed controller under artificially-induced track degradation.
Specifically, the experiments quantify the extent of degradation that our controller can effectively manage and \revisiontwo{demonstrate} its advantage over baseline controllers in similar scenarios.
We compare three controllers: (a) nonlinear PD (\eqref{eq:main_controller} without the adaptation), (b) MAGIC with constant $\boldsymbol{\Phi}$ in \eqref{eq:main_controller} \revisiontwo{and}~\eqref{eq:adapt_law}, \revisiontwo{and} (c)~\namepaper~with \ac{dnn} $\boldsymbol{\Phi}$ in \eqref{eq:main_controller}~\revisiontwo{and}~\eqref{eq:adapt_law}. 
The \ac{dnn} is not retrained on the new ground \revisiontwo{type}, but the previously trained \ac{dnn} from~\Cref{sec:experiments_mars_slopes_JPL} is employed.

\begin{figure}[t]
    \centering
\includegraphics[width=\columnwidth]{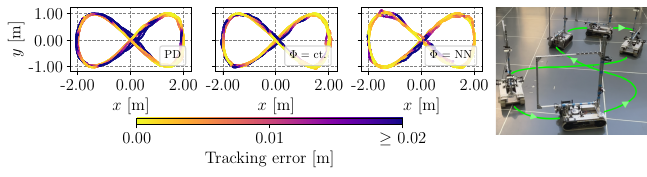}
    \caption{a) Tracking error for the three controllers during track degradation. \revision{(left)} the performance \revision{for the} nonlinear PD (the baseline). \revision{(middle and right) performance for the constant basis function and the terrain-informed basis function.}  b) \revisiontwo{The} figure 8s \revisiontwo{trajectories} for evaluating track degradation performance conducted indoors at CAST.  The consistent floor of CAST ensures any slippage is consistent both within the figure 8 and between tests.}
    \label{fig:cast_experiment}
\end{figure}

To simulate track degradation, a scalar factor is applied to one track that reduces its commanded rotation speed downstream of (and opaquely to) the controller. In this case, we apply a 70\% reduction in speed to the right track using a step function with a period of 3 seconds, while keeping the left track operating `nominally.'
The {GVR-Bot} is commanded to follow a figure 8 trajectory, and the results are shown in Fig.~\ref{fig:cast_experiment}\revisiontwo{,} with the \revisiontwo{ \ac{rmse}} in position tabulated in Table~\ref{tab:tracking_error_fault}.
The results show that both constant $\boldsymbol{\Phi}$ and \ac{dnn} $\boldsymbol{\Phi}$ controllers outperform the tracking of the baseline PD controller by 23\% and 31\%, proving the robustness to model mismatch of both DNN and non-\ac{dnn} controllers.

\begin{table}[t]
\caption{Position tracking error statistics \revision{for the tracked vehicle experiencing} track degradation.}
    \centering
    \begin{tabular}{c|c|c}
    \hline
    Controller &  \revision{Tracking RMSE} \revision{[m]} & \revision{Improvement} \\ \hline \hline
    \revision{(a) nonlinear} PD & $0.102$ & \revision {-} \\ \hline
    \revision{(b)} $\boldsymbol{\Phi}$ = constant & $0.079$  & \revision{23\%} \\
    \hline
    \revision{(c)}  $\boldsymbol{\Phi}$ = \ac{dnn} & $0.070$  & \revision{31\%}\\
    \hline
    \end{tabular} \label{tab:tracking_error_fault}
\end{table}

\subsection{Performance at DARPA's Learning Introspection Control}

\begin{figure}[t]
    \centering
    \includegraphics[width=0.9\columnwidth]{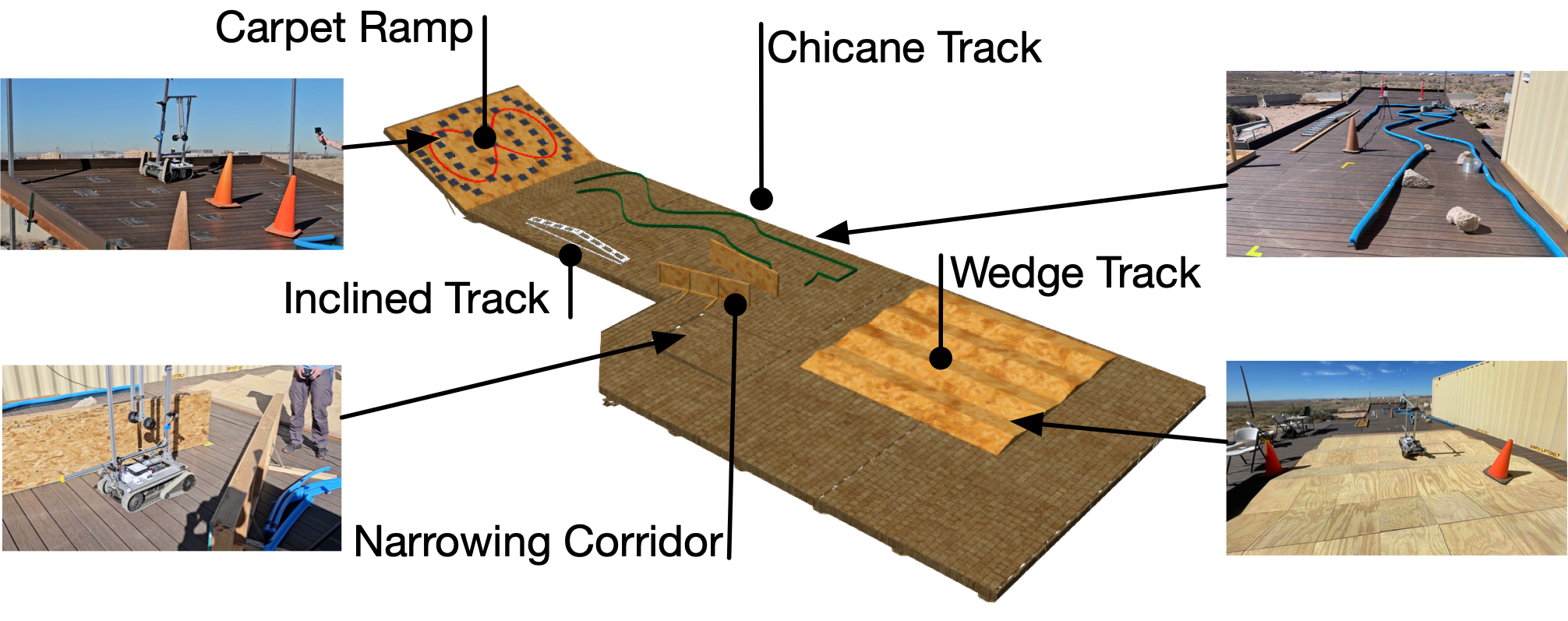}
    \caption{Combined Circuit for the DARPA \ac{linc} runs \revision{showing} the full course with break outs of each of the elements. Track credit: Sandia National Laboratories team.}
    \label{fig:combined_track}
\end{figure}

The DARPA \ac{linc} program~\cite{linc-website} develops machine learning-based introspection technologies that enable systems to respond to changes not predicted at design time.
\ac{linc} took place throughout 2023 at Sandia National Laboratories.

The main exercise, Combined Circuit (Fig. \ref{fig:combined_track}), \revisiontwo{evaluated conditions such as track degradation, collisions, tip-over}, and reduced cognitive load on the driver across a variety of test elements. 
Importantly, these exercises were completed with a human driver as the global planner \revisiontwo{in order to introduce} additional challenges such as adversarial driving and driver intent \revision{inference}.

For this exercise, we implemented the MAGIC controller from~\eqref{eq:main_controller},~\eqref{eq:adapt_law} in which the basis function~\eqref{eq:constant_phi} \revision{was} constant.
Trajectories (both position and velocity) were generated using a sampling-based motion planner based on \ac{mcts}, with the desired goal locations generated using a `driver intent' module that \revisiontwo{generated} a desired path based on operator joystick inputs.
The main modules of the software stack and their interfaces are shown in Fig.~\ref{fig:arch_darpa}, with our MAGIC controller highlighted in blue.

\begin{figure}[t]
    \centering
    \includegraphics[width=0.9\columnwidth]{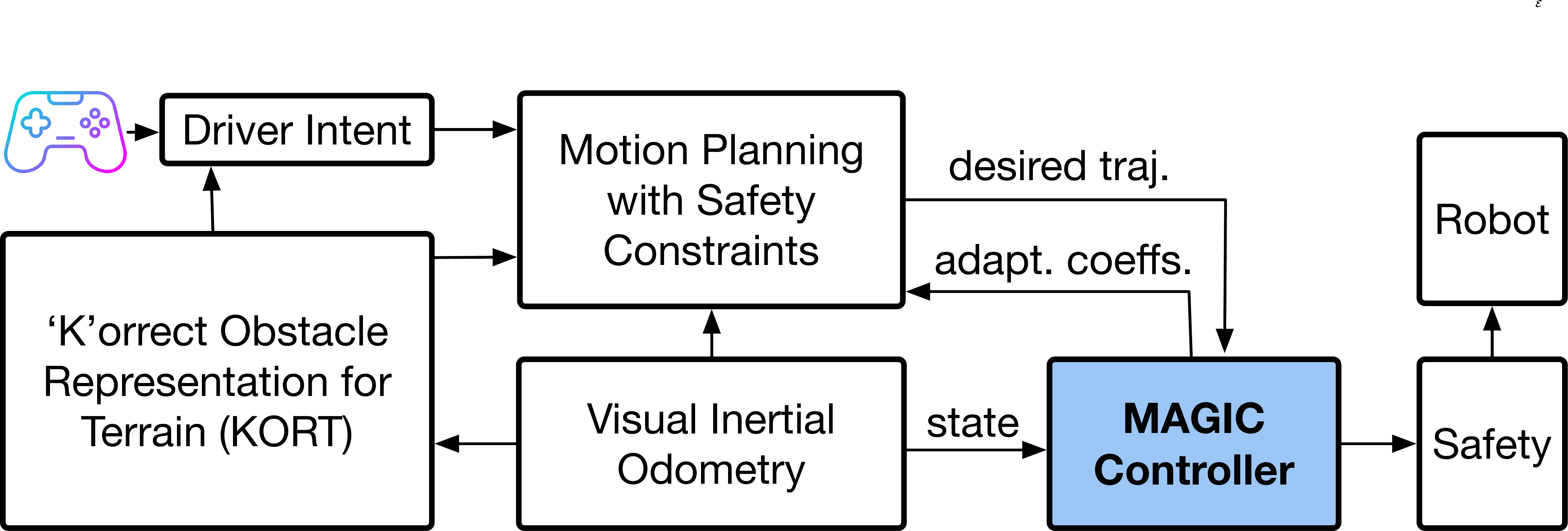}
    \caption{Architecture for our DARPA \ac{linc} software stack, \revision{with the MAGIC controller showcased in blue.}}
    \label{fig:arch_darpa}
\end{figure}

To evaluate the performance of our MAGIC controller, we compare the estimated state (linear and angular velocit\revisiontwo{ies}) from the \ac{vio} with the reference trajectory $\mathbf{v}_{\mathrm{ref}}$ computed from the desired trajectories generated by the \ac{mcts} planner, as explained in~\Cref{sec:adapt_control_tracked_vehicle}.
For the baseline, we compare the desired command from the joystick with the actual state from the \ac{vio}.

The following subsections discuss each of the components of the Combined Circuit and \revision{the performance of our controller}.
In Table~\ref{tab:performance_metrics_linc}, we present the performance metrics for the four exercises of the \ac{linc} project.
\revision{Each} exercise was traversed 4 times and the \revisiontwo{\ac{rmse}} of the linear and angular velocity was computed.

\begin{table*}[t]
    \caption{Performance metrics for the four exercises of the DARPA \ac{linc} project.}
    \label{tab:performance_metrics_linc}
\begin{tabular}{p{1.32cm}|p{1cm}p{1cm}p{1cm}|p{1cm}p{1cm}p{1cm}|p{1cm}p{1cm}p{1cm}|p{1cm}p{1cm}p{1cm}}

                       & \multicolumn{3}{c}{Chicane}                                                                                          & \multicolumn{3}{c}{Carpet Ramp}                                                                                      & \multicolumn{3}{c}{Wedges}                                                                                         & \multicolumn{3}{c}{Narrow Corridor}                                                                                \\ \hline
                       & \multicolumn{1}{p{0.85cm}|}{$v$ error {[}m/s{]}}       & \multicolumn{1}{p{0.85cm}|}{$\omega$ error {[}rad/s{]}}       & Time {[}s{]} & \multicolumn{1}{p{0.85cm}|}{$v$ error {[}m/s{]}}       & \multicolumn{1}{p{0.85cm}|}{$\omega$ error {[}rad/s{]}} & Time {[}s{]} & \multicolumn{1}{p{0.85cm}|}{$v$ error {[}m/s{]}} & \multicolumn{1}{p{0.85cm}|}{$\omega$ error {[}rad/s{]}} & Time {[}s{]} & \multicolumn{1}{p{0.85cm}|}{$v$ error {[}m/s{]}} & \multicolumn{1}{p{0.85cm}|}{$\omega$ error {[}rad/s{]}} & Time {[}s{]} \\ \hline \hline
    \ac{linc} off               & \multicolumn{1}{p{0.85cm}|}{0.28}                      & \multicolumn{1}{p{0.85cm}|}{\revision{0.45}}                             & 16.36        & \multicolumn{1}{p{0.85cm}|}{\revision{0.96}}                      & \multicolumn{1}{p{0.85cm}|}{\revision{0.59}}                             & 19.96        & \multicolumn{1}{p{0.85cm}|}{\revision{0.37}}                      & \multicolumn{1}{p{0.85cm}|}{\revision{0.58}}                           & 34.95        & \multicolumn{1}{p{0.85cm}|}{\revision{0.52}}                      & \multicolumn{1}{p{0.85cm}|}{\revision{0.72}}                           & 17.95        \\ 
\ac{linc} on                & \multicolumn{1}{p{0.85cm}|}{0.16}                      & \multicolumn{1}{p{0.85cm}|}{\revision{0.36}}                             & 44.36        & \multicolumn{1}{p{0.85cm}|}{0.36}                      & \multicolumn{1}{p{0.85cm}|}{\revision{0.32}}                             & 39.22        & \multicolumn{1}{p{0.85cm}|}{\revision{0.25}}                      & \multicolumn{1}{p{0.85cm}|}{\revision{0.34}}                           & 39.72        & \multicolumn{1}{p{0.85cm}|}{\revision{0.19}}                      & \multicolumn{1}{p{0.85cm}|}{\revision{0.44}}                           & 23.73        \\ \hline 
    Improvement            & \multicolumn{1}{p{0.85cm}|}{\textbf{\revision{42}\%}}          & \multicolumn{1}{p{0.85cm}|}{\textbf{\revision{19}\%}}                 &              & \multicolumn{1}{p{0.85cm}|}{\textbf{\revision{62}\%}}          & \multicolumn{1}{p{0.85cm}|}{\textbf{\revision{46}\%}}                 &              & \multicolumn{1}{p{0.85cm}|}{\textbf{\revision{33}\%}}          & \multicolumn{1}{p{0.85cm}|}{\textbf{\revision{41}\%}}               &              & \multicolumn{1}{p{0.85cm}|}{\textbf{\revision{64}\%}}                   & \multicolumn{1}{p{0.85cm}|}{\textbf{\revision{39}\%}}      &              \\  \cline{1-3} \cline{5-6} \cline{8-9} \cline{11-12}
\end{tabular}
\end{table*}

\subsubsection{Chicane Track}

The Chicane Track highlighted the rejection of artificially induced track degradation, which was applied dynamically and opaquely as the {GVR-Bot} traversed the course.
Due to the narrow track (the width is 0.9~m on average, 0.25~m wider than the {GVR-Bot} on both sides), track degradation leads to an increase \revision{in collisions} with the chicane walls if not quickly adapted to.
Our MAGIC controller was able to successfully adapt to these challenges, thus making this artificially induced track degradation almost imperceptible to the driver after a very short initial adaptation transient. 

\begin{figure}[t]
    \centering
    \includegraphics[width=0.9\columnwidth]{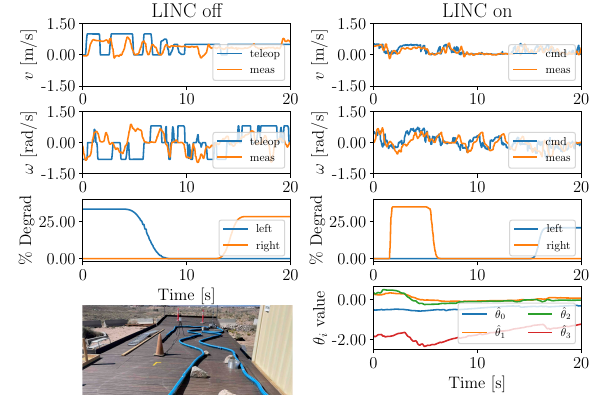}
    \caption{(Chicane Track) The \revisiontwo{left} and \revisiontwo{right} columns display tracking performance with\revisiontwo{out} and with our MAGIC controller, respectively.}
    \label{fig:linc-comparison-chicane}
\end{figure}

The effectiveness of the trajectory tracking on the Chicane Track is shown in Fig.~\ref{fig:linc-comparison-chicane} for both the baseline and the MAGIC controller. In the first two rows, the tracking of the velocities is emphasized. The third row shows the amount of degradation applied to the system. %
The bottom plot shows the estimated adaptation parameters $\hat{\vtheta}$ changing in real time to compensate for the track degradation. ~\Cref{tab:performance_metrics_linc} \revision{shows} the improved performance of the MAGIC controller on this exercise.
Our controller improved \revision{linear velocity} tracking by 42\%, \revision{and} angular velocity \revision{tracking} by 19\%.
Because the track degradation information, $\mathbf{B}_n+\mphi\hat{\boldsymbol{\theta}}$ of \eqref{eq:main_controller}, is estimated by the MAGIC controller in real-time, the \ac{mcts} can successfully generate trajectories that use this corrected control matrix, thereby successfully avoiding collisions with the chicane walls. 
When MAGIC was activated, the robot navigated the chicane track more cautiously, moving approximately 2.5 times slower than with the baseline controller. This reduction in pace was a result of the software stack prioritizing safety.

\subsubsection{Carpet Ramp} \label{sec:carpet_ramp}
The goal of the Carpet Ramp exercise is to restore and maintain control under \emph{track degradation} and \emph{variable slippage}, all whilst mitigating the risk of tipping over.
The ramp \revision{had a} slippery wooden surface with several patches of carpet to alter the ground friction coefficient, causing the tracks to slip asymmetrically. 
Additionally, as the roll angle of the robot increases over the incline, the traction of one of its tracks is reduced as more of the weight falls over one of the tracks due to the high vertical center of gravity.
This imbalance in traction causes the dynamics of the {GVR-Bot} to change significantly, especially affecting the ability to turn. 
This restricted turning behavior is shown in~\Cref{fig:linc-comparison-carpet-ramp}.
The plot in the first column, second row shows that although the operator attempts to turn the {GVR-Bot}, very little control authority in angular velocity is achieved.
By comparison, when operated with our MAGIC controller, the robot adapts to the terrain, tracking safer turn commands that reduce the risk of tipping (second column, second row).
As seen in the bottom row of Fig.~\ref{fig:linc-comparison-carpet-ramp}, the adaptation coefficients, especially the one for the angular velocity, greatly increase to compensate for slip.
This particular exercise demonstrates the greatest improvement in performance relative to the baseline,
as seen in Table~\ref{tab:performance_metrics_linc}.

\begin{figure}[t]
    \centering
\includegraphics[width=0.9\columnwidth]{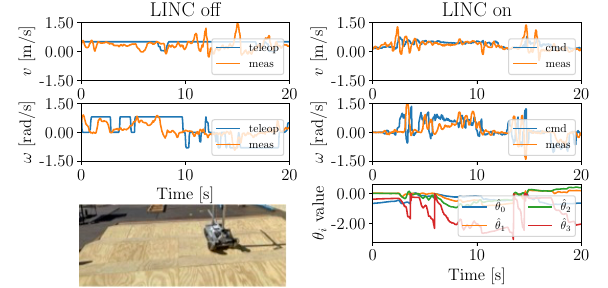}
    \caption{(Carpet Ramp) The \revisiontwo{left} and \revisiontwo{right} columns display tracking performance with\revisiontwo{out} and with our MAGIC controller. }
    \label{fig:linc-comparison-carpet-ramp}
\end{figure}

\subsubsection{Narrowing Corridor}

The aim of the Narrowing Corridor mirrored that of the Chicane Track, \revisiontwo{assesing} the robot's \revisiontwo{ability to consistently navigate} through a tight corridor despite track degradation.
\revisiontwo{For the robot's} performance, see Table~\ref{tab:performance_metrics_linc}.

\subsubsection{Wedge Track}
\revision{Similar to the Carpet Ramp exercise in~\Cref{sec:carpet_ramp}}, the Wedge Track tests the ability of the algorithms to maintain control and slow down under slippage while minimizing \revisiontwo{the risk of} tipping.
When traversing discrete wooden wedges, the robot often loses traction.
Moreover, the robot experiences sudden positive and negative accelerations \revision{due to} the downhill and uphill traversal of a wedge pair.
\begin{figure}[t]
    \centering
    \includegraphics[width=0.9\columnwidth]{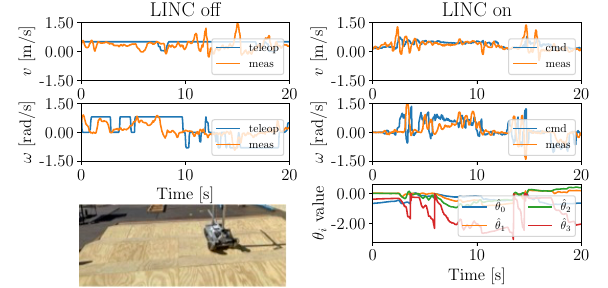}
    \caption{Trajectory tracking during the Wedge Track run. The \revisiontwo{right} and \revisiontwo{left} columns display tracking performance with\revisiontwo{out} and with our controller. No artificial degradation was introduced by the organizers.}
    \label{fig:linc-comparison-wedges}
\end{figure}
As shown in~\Cref{fig:linc-comparison-wedges}, with our controller's assistance and safe slowdowns from the planning, the robot can track velocities more accurately than without our MAGIC controller.
By comparison, without assistance, the robot experiences large velocity spikes as it traverses the wedges. These rapid changes are caused by the \ac{vio}'s Kalman filter \revisiontwo{that integrates} spikes measured by the accelerometer when the robot bounces off the wedges.
The bottom row of~\Cref{fig:linc-comparison-wedges} \revision{shows} the adaptation coefficients quickly adapting for the loss of traction.

\subsubsection{MAGIC and Human-in-the-Loop}
The \ac{linc} program was different from many robotics projects in that the global planner was human-driven rather than autonomous.
This presents a challenge as MAGIC must not degrade the user driving experience~\revisiontwo{but} instead must augment it without the forward-planning and control input smoothness assumptions of typical robotic projects.
The success of MAGIC in augmenting a human driver was twofold\revisiontwo{:} firstly, MAGIC consistently ran fast enough such that there was no perceptible increase between joystick input and robot, and secondly, much of the adaptation to the changing terrain and \revisiontwo{vehicle} were significantly reduced by MAGIC.

\section{Conclusion}\label{sec:conclusions}

We introduced a novel learning-based composite adaptive controller that incorporates visual foundation models for terrain understanding and adaptation.
The basis function of this adaptive controller, which is both state and terrain dependent, is learned offline using our proposed meta-learning algorithm. 
We prove the exponential convergence to a bounded tracking error ball of our adaptive controller and demonstrate that incorporating a pre-trained \ac{vfm} into our learned representation enhances our controller's tracking performance compared to an equivalent controller without the learned representation. 
Our method showed a 53\% \revisiontwo{decrease} in position tracking error when deployed on a tracked vehicle traversing two different sloped terrains.
\revision{We further demonstrated our algorithm on-board a car-like vehicle and showed \revisiontwo{that} the learnt DNN basis function \revisiontwo{captures} the residual dynamics generated by the two different terrains.}

To gain insight into the inner workings of our full method,
we empirically analyzed the features of the pre-trained \ac{vfm} in terms of separability and continuity using support vector classifiers.
This analysis showed positive empirical evidence
that the DINO \ac{vfm} is suitable for fine-grained discrimination
of terrain types in images containing only terrain, and thus suitable for our control method.

We further tested our method under other perturbations, such as artificially induced track degradation.
We demonstrated the effectiveness of our algorithm without terrain-aware basis function in human-in-the-loop driving scenarios. 
Our controller improved tracking of real-time human generated trajectories both in nominal and degraded vehicle states without introducing noticeable system delay as part of the DARPA's \ac{linc} project.

\section*{Acknowledgement}
We thank T. Touma for developing the hardware and the testing environments replicas for the tracked vehicle;
L. Gan and P. Proença for the state estimation; E. Sj\"ogren for her early work on DINO features;
Sandia National Laboratories team (T. Blada, D. Wood, E. Lu) for organizing the \ac{linc} tracks;
J. Burdick for the \ac{linc} project management; A. Rahmani for the JPL Mars Yard support, \ac{linc} project management, and technical advice, and
Y. Yue, B. Riviere, and P. Spieler for stimulating discussions. Some hardware experiments were conducted at Caltech's CAST.

\bibliographystyle{ieeetr}
\bibliography{reference}

\begin{IEEEbiography}[{\includegraphics[width=1in,height=1.25in,clip,keepaspectratio]{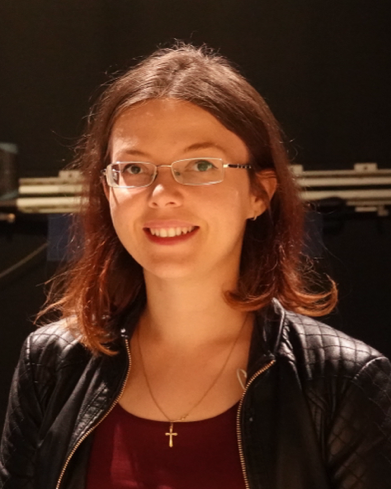}}]{Elena Sorina Lupu} is a PhD candidate in the Aerospace Department at California Institute of Technology and an affiliate of the Keck Institute of Space Sciences. She has master degrees from Caltech in Space Engineering and from the Swiss Federal Institute of Technology, Lausanne (EPFL) in Robotics and Autonomous Systems. Her current research focuses on autonomy, control, and machine learning. 
She has won awards and fellowships including the AI4Science Fellowship and the Anita Borg Women Techmakers.
\end{IEEEbiography}
\vspace{-1.5cm}
\begin{IEEEbiography}[{\includegraphics[width=1in,height=1.25in,clip,keepaspectratio]{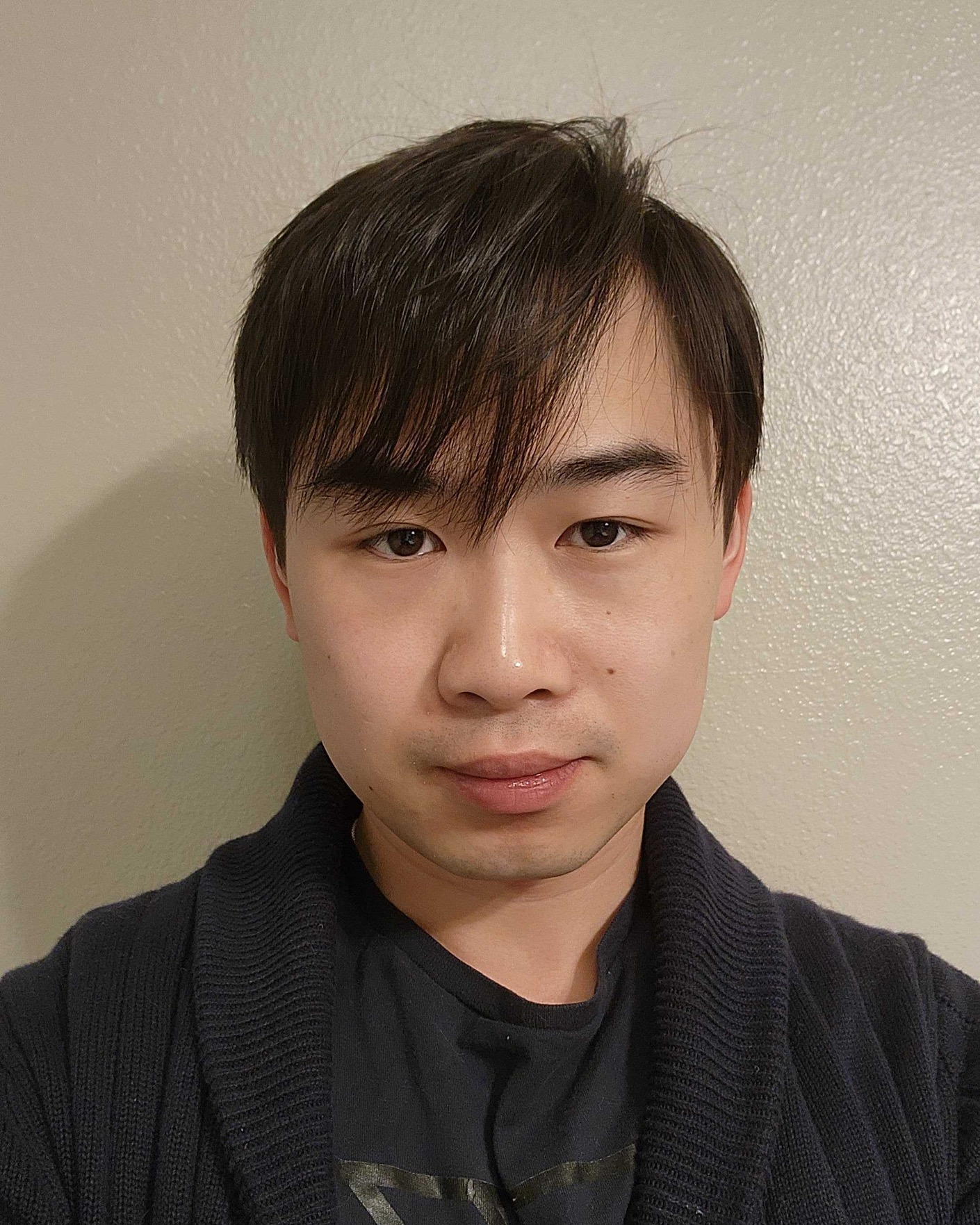}}]{Fengze Xie} is a PhD student in the Department of Computing and Mathematical Sciences at the Caltech. He received a M.S. in Electrical Engineering from the California Institute of Technology in 2021, following the completion of a B.S. in Computer Engineering and a B.S. in Engineering Physics from the University of Illinois Urbana-Champaign in 2020. His research interests mainly focus on the intersection between machine learning and robotics, including control and planning. He was the recipient of 
the Simoudis Discovery Prize. 
\end{IEEEbiography}
\vspace{-1.5cm}
\begin{IEEEbiography}[{\includegraphics[width=1in,height=1.25in,clip,keepaspectratio]{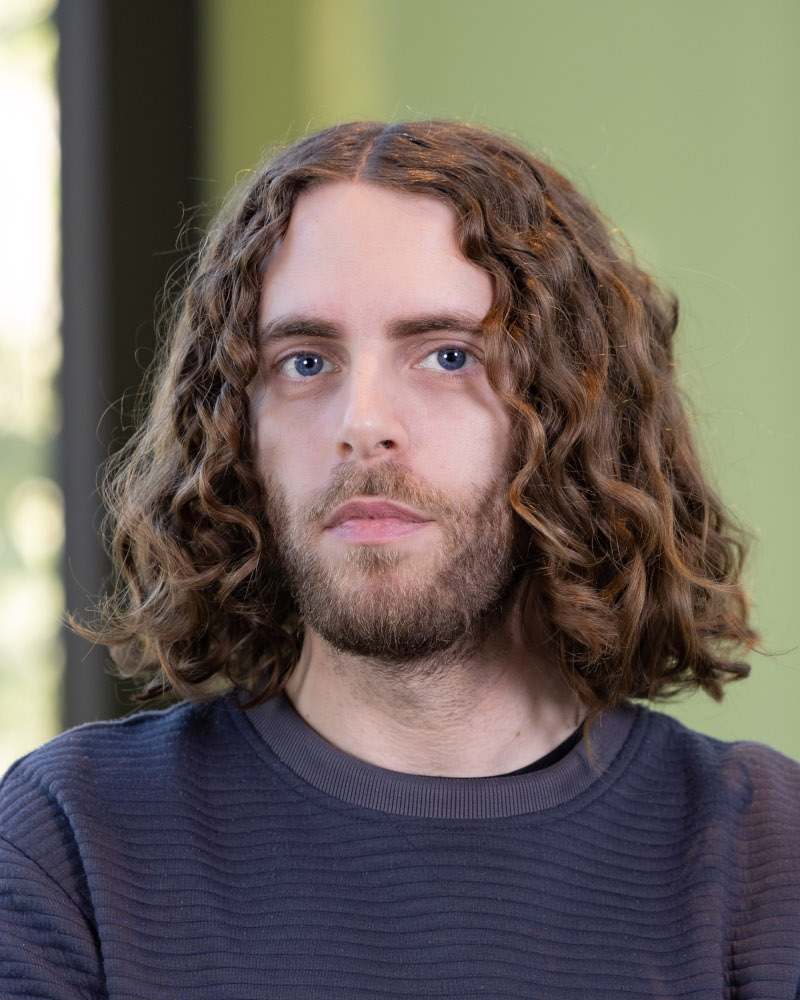}}]{James Alan Preiss}
received the Ph.D.\ degree in Computer Science from the University of Southern California in 2022
and the B.A./B.S.\ degree with concentrations in mathematics and photography from The Evergreen State College in 2010.
He was a Postdoctoral Scholar with the Department of Computing + Mathematical Sciences at Caltech, Pasadena, CA, USA. He is currently an Assistant Professor with the Department of Computer Science at the University of California, Santa Barbara, CA, USA.
His research interests focus on developing rigorous foundations for the application of machine learning to planning and control in robotics.
\end{IEEEbiography}

\vspace{-1.5cm}

\begin{IEEEbiography}[{\includegraphics[width=1in,height=1.25in,clip,keepaspectratio]{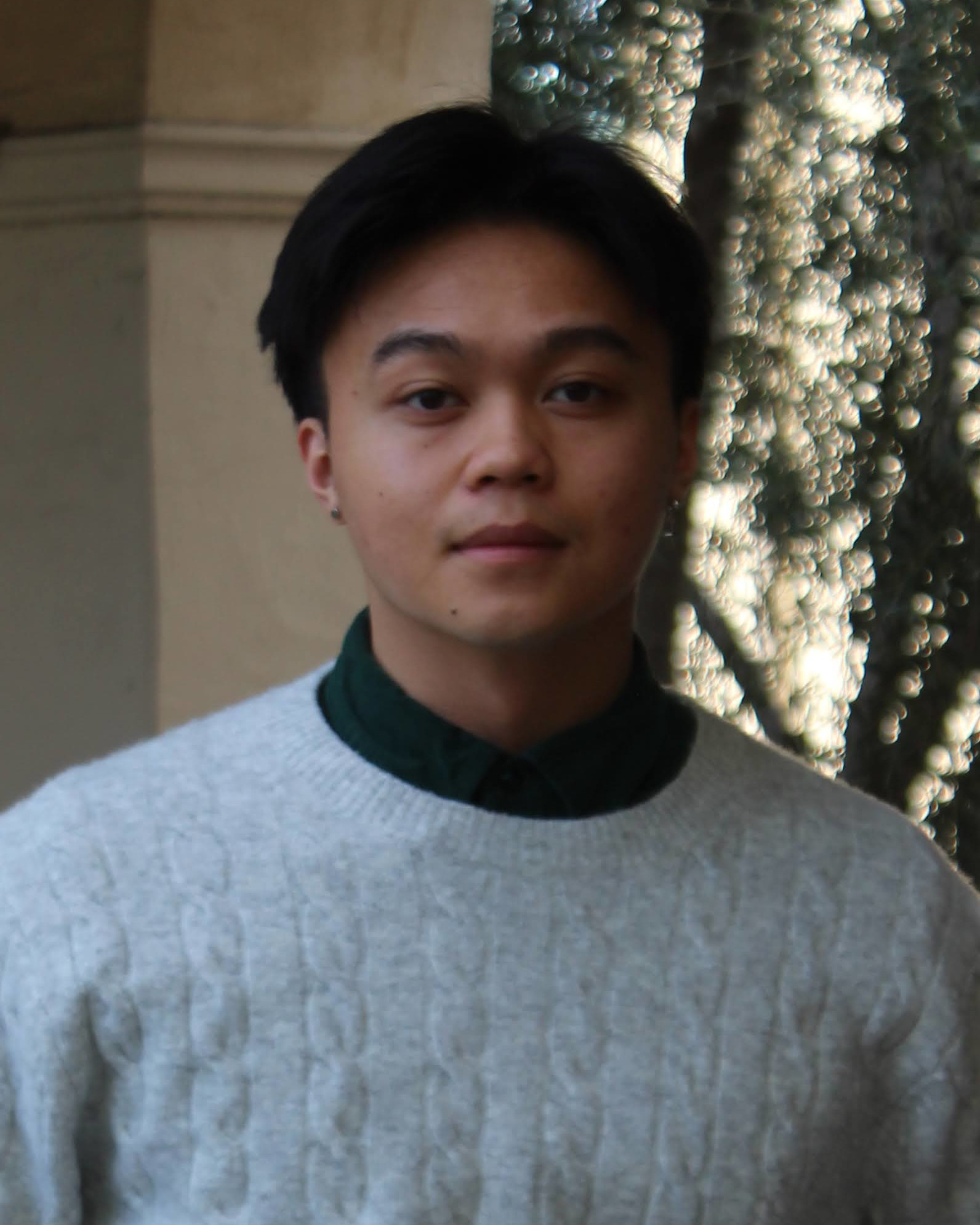}}]{Jedidiah Alindogan} is a Research Engineer at Caltech. He received his B.S. in mechanical engineering from Caltech in 2023. At Caltech, he along with a team of undergraduates submitted to the NASA 2022 BIG Idea challenge and was awarded the Best Visionary Concept Award. Finally, he was granted membership to the Tau Beta Phi engineering honor society for academic achievement as well as a commitment to personal and professional integrity.
\end{IEEEbiography}

\vspace{-1.5cm}

\begin{IEEEbiography}[{\includegraphics[width=1in,height=1.25in,clip,keepaspectratio]{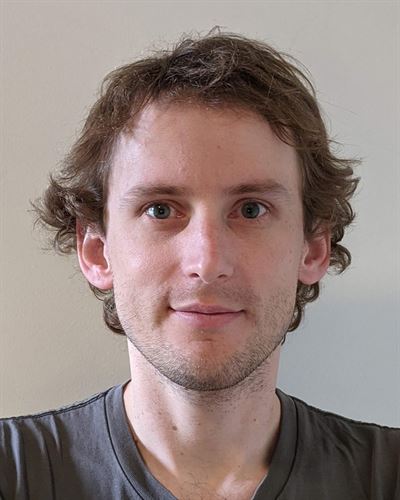}}]{Matthew Anderson}
received a B.Eng (Aeronautical, Hons I) from The University of Sydney in 2010 and a jointly awarded Ph.D from The University of Sydney and Universit\'e libre de Bruxelles in 2018.  
He is currently a Staff Scientist in the Computing + Mathematical Sciences (CMS) department at the California Institute of Technology and supports a wide variety of robotics research with a fleet of UAVs and UGVs. 
His primary areas of research focus on UAVs and extreme environment robotics. 
\end{IEEEbiography}

\vspace{-1.5cm}

\begin{IEEEbiography}[{\includegraphics[width=1in,height=1.25in,clip,keepaspectratio]{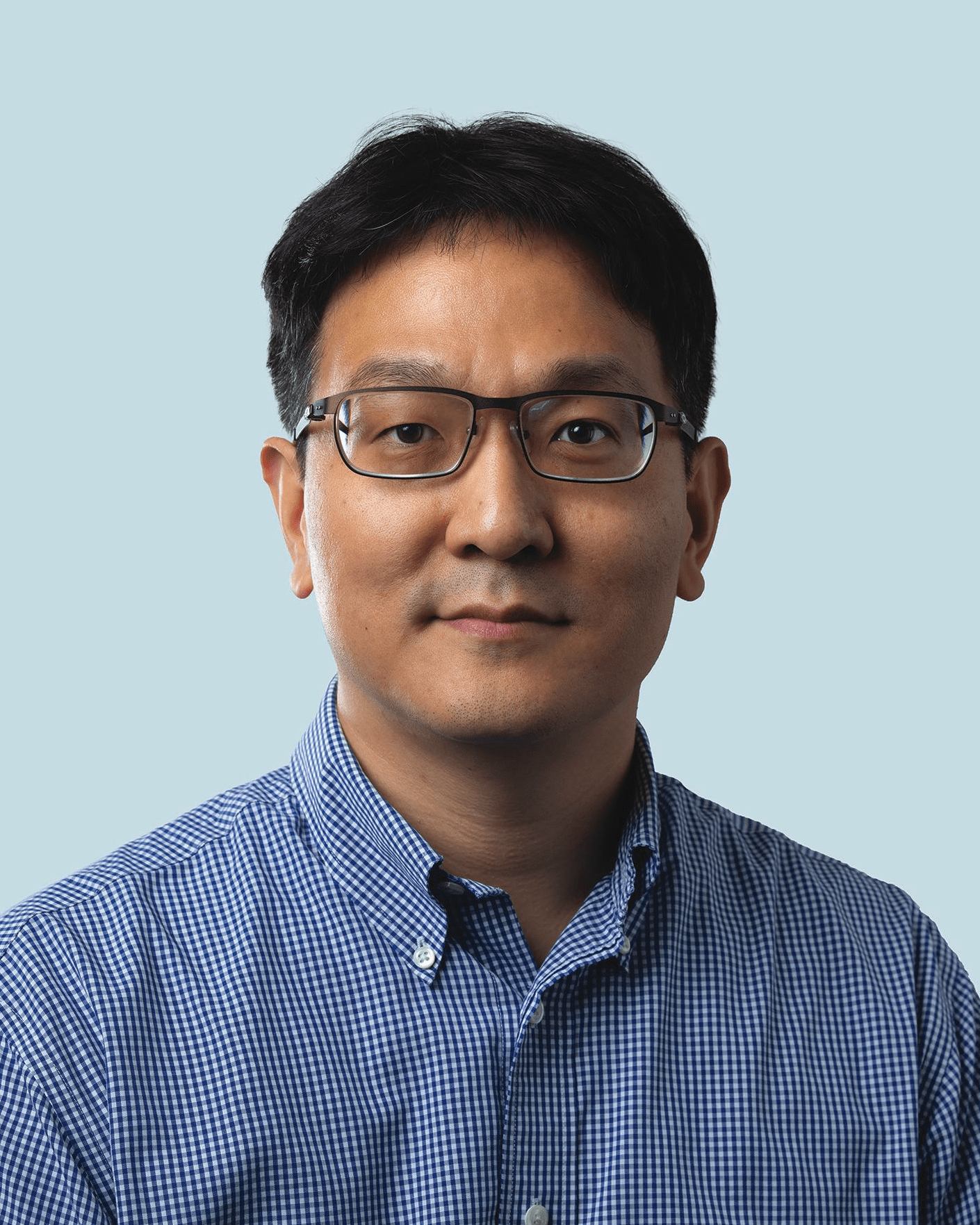}}]{Soon-Jo Chung} (M’06–SM’12) received the S.M. degree in aeronautics and astronautics and the Sc.D. degree in estimation and control from Massachusetts Institute of Technology, Cambridge, MA, USA, in 2002 and 2007, respectively. He is currently Bren Professor of Control and Dynamical Systems, and a Jet Propulsion Laboratory Senior Research Scientist at the California Institute of Technology, Pasadena, CA, USA. %
He was an Associate Editor of  IEEE Transactions on Automatic Control and IEEE Transactions on Robotics, and the Guest Editor of a Special Section on Aerial Swarm Robotics published in IEEE Transactions on Robotics.
\end{IEEEbiography}

\end{document}